\documentclass[conference]{IEEEtran}
\pdfoutput=1

\usepackage[utf8]{inputenc}
\usepackage{amsmath}
\usepackage{amsthm}
\usepackage{amsfonts}
\usepackage{stmaryrd}
\usepackage[noend]{algpseudocode}
\usepackage[ruled]{algorithm}
\usepackage{multirow}
\usepackage{todonotes}
\usepackage{comment}

\usepackage{caption}
\usepackage{subcaption}

\newcommand{\feats}{\mathcal{X}}
\newcommand{\labels}{\mathcal{Y}}
\newcommand{\R}{\mathbb{R}}
\newcommand{\dtrain}{\mathcal{D}_{\textit{train}}}
\newcommand{\dtest}{\mathcal{D}_{\textit{test}}}
\newcommand{\drand}{\mathcal{D}_{\textit{rand}}}
\newcommand{\dany}{\mathcal{D}}

\newcommand{\writtenbyLC}[1]{#1}
\newcommand{\flip}[1]{\delta(#1,S)}

\newcommand{\ccd}{\tilde{d}}

\newcommand{\interp}[1]{\llbracket #1 \rrbracket}
\newcommand{\empirical}[1]{#1}

\PassOptionsToPackage{hyphens}{url}\usepackage[hidelinks]{hyperref}

\theoremstyle{definition}
\newtheorem{definition}{Definition}
\newtheorem{theorem}{Theorem}

\title{Explainable Global Fairness Verification of Tree-Based Classifiers}

\author{\IEEEauthorblockN{Stefano Calzavara\IEEEauthorrefmark{1}\textsuperscript{1},
Lorenzo Cazzaro\IEEEauthorrefmark{1}\textsuperscript{1}, Claudio Lucchese\IEEEauthorrefmark{1}\textsuperscript{1} and
Federico Marcuzzi\IEEEauthorrefmark{1}\textsuperscript{1}}
\IEEEauthorblockA{\IEEEauthorrefmark{1}Department of Environmental Sciences, Informatics and Statistics,\\
Ca’ Foscari University of Venice, Italy.\\
Email: \{name.surname\}@unive.it}
}

\begin{document}

\maketitle

\begingroup\renewcommand\thefootnote{1}
\footnotetext{Equal contribution}
\endgroup

\begin{abstract}
We present a new approach to the global fairness verification of tree-based classifiers. Given a tree-based classifier and a set of sensitive features potentially leading to discrimination, our analysis synthesizes sufficient conditions for fairness, expressed as a set of traditional propositional logic formulas, which are readily understandable by human experts. The verified fairness guarantees are global, in that the formulas predicate over all the possible inputs of the classifier, rather than just a few specific test instances. Our analysis is formally proved both sound and complete. Experimental results on public datasets show that the analysis is precise, explainable to human experts and efficient enough for practical adoption.
\end{abstract}

\section{Introduction}
The ever-increasing success of Machine Learning (ML) led to its massive deployment in a range of different settings over the last years. Unfortunately, classic approaches to assess the performance of ML models do not always provide a reliable picture of their effectiveness in so many varied practical deployments. For example, traditional deep learning models with state-of-the-art accuracy may be vulnerable to evasion attacks, i.e., small perturbations of test inputs designed to force prediction errors, thus making their deployment in adversarial settings unfeasible~\cite{SzegedyZSBEGF13,MadryMSTV18}. Similarly, ML models may turn out to be \emph{unfair} for automated decision-making. For example, a commercial recidivism-risk assessment algorithm was found to be racially biased~\cite{LarsonMKA16} and an existing algorithm adopted in the US falsely determined that black patients were healthier than other patients with similar conditions~\cite{ObermeyerM19}. These incidents led to a proliferation of research on fair ML in recent times, as summarized in different surveys~\cite{abs-2010-04053,abs-2012-15816,MehrabiMSLG21}.

Fairness in ML has been analyzed from different angles and can be broadly categorized in two main research lines. The first one includes the development of new ML algorithms that are able to mitigate the bias that is directly
or indirectly present in the training data~\cite{AghaeiAV19,RanzatoUZ21,RuossBFV20}. The second, complementary research line investigates techniques to estimate or even formally verify the fairness guarantees provided by existing ML models~\cite{JohnVS20,UrbanCWZ20,IgnatievCSHM20}. This paper contributes to the latter line of work, which is still at an early stage of development and suffers from relevant shortcomings.

A very popular approach to assess the fairness guarantees of ML models is based on \emph{testing}~\cite{GalhotraBM17,AggarwalLNDS19, UdeshiAC18, BlackYF20}. The key common intuition underlying any fairness testing strategy is straightforward: generating a number of test inputs designed to automatically identify individuals who may suffer from discrimination by the ML model. Unfortunately, as for any testing approach, this type of analysis is under-approximated: these proposals can identify room for unfair treatment, but cannot establish formal fairness proofs. This is sub-optimal because it does not allow one to prove that unfair behavior can never affect specific classes of individuals. For this reason, recent papers advocated the adoption of formal \emph{fairness verification} techniques to prove lack of discrimination~\cite{RuossBFV20,JohnVS20,UrbanCWZ20,IgnatievCSHM20,abs-2205-09927}. Most of the work in the area, however, just focuses on deep neural networks and disregards tree-based ML models, such as decision trees~\cite{BreimanFOS84} and random forests~\cite{Breiman01}, which are still exceptionally popular, in particular for non-perceptual classification tasks. The only notable fairness verification approach designed for tree-based classifiers leverages abstract interpretation to verify \emph{local fairness} properties~\cite{RanzatoUZ21}. Unfortunately, local fairness is now recognized as a rather weak property predicating just on specific test instances, while \emph{global fairness} predicates over all the possible inputs of the classifier and is thus more reliable to assess the actual fairness guarantees that it provides~\cite{abs-2205-09927}.

\subsection{Contributions}
We present a new approach to the global fairness verification of tree-based classifiers. Given a tree-based classifier and a set of sensitive features potentially leading to discrimination, our analysis synthesizes sufficient conditions for fairness, expressed as a set of traditional propositional logic formulas $F$ predicating over the entire feature space, rather than just on a specific test set, thus providing global guarantees.

Our fairness verification approach is formally proved both \emph{sound} and \emph{complete}, i.e., fairness is certified for any instance satisfying some formula in $F$, and the formulas in $F$ can characterize all the instances where the classifier is fair. Moreover, our approach is \emph{explainable}, i.e., it is readily understandable by human experts, being based on traditional logic formulas. In particular, we empirically show that a small set of simple logic formulas suffices to largely characterize the fairness guarantees provided by the classifier in practice. This makes our approach particularly appealing for problems like algorithmic hiring, where automated decisions must be carefully audited~\cite{SchumannFMD20}.

\subsection{Structure of the Paper}
The paper is organized as follows:
\begin{itemize}
    \item Section~\ref{sec:background} reviews the background and introduces the necessary ingredients to appreciate the technical contribution;
    \item Section~\ref{sec:analysis} presents our fairness verification approach and establishes the formal guarantees provided by our analysis. Moreover, it describes the implementation of the analysis in C++. We plan to release the developed software upon paper acceptance;
    \item Section~\ref{sec:experiments} reports on our experimental evaluation on public datasets. Our experiments assess the precision of the analysis, its explainability, and its performance;
    \item Section~\ref{sec:related} presents and compares with related work, thus clarifying the distinctive features of our proposal;
    \item Section~\ref{sec:conclusion} briefly concludes the paper and describes possible future work in the area.
\end{itemize}
\section{Background}
\label{sec:background}

We here introduce the key technical ingredients required to appreciate the contribution of the paper.

\subsection{Tree-Based Classifiers}
Let $\feats \subseteq \R^d$ be a $d$-dimensional vector space of features and let $\labels$ be a finite set of class labels. We assume each element of the feature space $\vec{x} = \langle x_1,\ldots,x_d \rangle$, called \emph{instance}, to be assigned a correct class $y$ by an unknown \emph{target} function $g: \feats \rightarrow \labels$. A \emph{classifier} is a function $f: \feats \rightarrow \labels$ intended to approximate the target function $g$ as accurately as possible. Normally, $f$ is automatically trained by a supervised learning algorithm, using a \emph{training set} $\dtrain = \{(\vec{x}_i,g(\vec{x}_i))\}_i$ of correctly labeled instances. The performance of $f$ is then assessed on a separate \emph{test set} $\dtest = \{(\vec{z}_i,g(\vec{z}_i))\}_i$, sampled from the same distribution of the training set.

In this paper, we focus on \emph{decision tree} classifiers~\cite{BreimanFOS84}. Decision trees can be inductively defined as follows: a decision tree $t$ is either a leaf $\lambda(y)$ for some label $y \in \labels$ or an internal node $\sigma(f,v,t_l,t_r)$, where $f \in \{1,\ldots,d\}$ identifies a feature, $v \in \R$ is a threshold for the feature, and $t_l,t_r$ are decision trees (left and right respectively). At test time, the instance $\vec{x}$ traverses the tree $t$ until it reaches a leaf $\lambda(y)$, which returns the prediction $y$, denoted by $t(\vec{x}) = y$. Specifically, for each traversed tree node $\sigma(f,v,t_l,t_r)$, $\vec{x}$ falls into the left sub-tree $t_l$ if $x_f \leq v$, and into the right sub-tree $t_r$ otherwise. Figure~\ref{fig:tree} represents an example decision tree of depth two, which assigns the label $+1$ to the instance $\langle 10,6 \rangle$ and the label $-1$ to the instance $\langle 6,9 \rangle$. Decision trees are normally combined into an \emph{ensemble} $T = \{t_1,\ldots,t_n\}$ to improve their predictive power~\cite{Breiman01}: in this case, the ensemble prediction $T(\vec{x})$ is computed by combining together the individual tree predictions $t_i(\vec{x})$, e.g., by performing majority voting on the individually predicted classes.

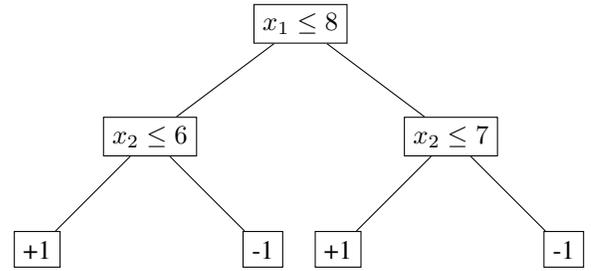
\begin{figure}[t]
\centering
\begin{tikzpicture}[level 1/.style={sibling distance=4cm},level 2/.style={sibling distance=3cm}]
\tikzstyle{every node}=[rectangle,draw]
\node{$x_1 \leq 8$}
	child { node {$x_2 \leq 6$}
	        child { node {+1}}
	        child { node {-1}} }
	child { node {$x_2 \leq 7$}
		    child { node {+1}}
	        child { node {-1}} }
;
\end{tikzpicture}
\caption{Example of decision tree}
\label{fig:tree}
\end{figure}

\subsection{Fairness in ML}
Many definitions of fairness have been proposed in the literature, each with its pros and cons~\cite{VermaR18}. In this paper, we consider (lack of) \emph{causal discrimination}~\cite{GalhotraBM17} as our fairness definition. We focus on this property for several reasons: its intuitive flavour, its popularity in the literature, and its independence from the distribution of the class labels (target function), which simplifies its practical application. Moreover, lack of causal discrimination is a powerful foundation for fairness, because it does not rely on the choice of a specific test set, but rather predicates over all the possible instances in (a subset of) the feature space $\feats$, much in line with recent proposals on the verification of \emph{global robustness} properties of machine learning models~\cite{Chen0QLJW21,LeinoWF21,CalzavaraCLMO21}. This is important in the fairness setting, because fairness is particularly relevant for minorities for which it might be hard to collect representative data in the test set. Indeed, the need for global fairness verification has been recently advocated for neural networks~\cite{abs-2205-09927}.

Intuitively, lack of causal discrimination requires that the classifier returns the same prediction on any two instances differing just for the value of a set of \emph{sensitive} features $S \subseteq \{1,\ldots,d\}$. Given an instance $\vec{x}$, we let $\flip{\vec{x}}$ stand for the set of the instances differing from $\vec{x}$ just for a (possibly empty) subset of the sensitive features $S$. For example, if $S$ included just two binary features, then $\flip{\vec{x}}$ would include the four instances obtained from $\vec{x}$ by setting each of the sensitive features in $S$ to one of their two possible values, while keeping the other features unchanged. Formally, lack of causal discrimination is then defined as follows.

\begin{definition}[Causal Discrimination]
\label{def:discrimination}
Let $f: \feats \rightarrow \labels$ be a classifier and let $S$ be a set of sensitive features. We say that $f$ does not perform \emph{causal discrimination} on $\feats' \subseteq \feats$ if and only if, for every instance $\vec{x} \in \feats'$, we have that $\forall \vec{z} \in \flip{\vec{x}}: f(\vec{z}) = f(\vec{x})$.
\end{definition}

To exemplify the definition at work, suppose a classifier is used to evaluate the legitimacy of loan requests and that the set of sensitive features $S$ just includes the customer's gender. Lack of causal discrimination on $\feats$ requires that any two identical customers differing just for their gender are guaranteed to get the same response to their loan requests. Focusing on a specific $\feats' \subseteq \feats$ allows one to make the fairness guarantees \emph{conditional} and thus more practically useful, e.g., by requiring that any two identical customers with a monthly salary higher than \$4,000 are guaranteed to receive the same response to their loan requests, irrespective of their gender.
\section{Global Fairness Verification}
\label{sec:analysis}

We here present our new approach to the global fairness verification of tree-based classifiers. 

\subsection{Overview}
Verifying the lack of causal discrimination is challenging, because it involves a universal quantification over a set of instances, possibly drawn from a continuous, unbounded feature space. Prior work on causal discrimination circumvented this problem by restricting its focus to a finite feature space and assessing fairness by means of a testing approach~\cite{GalhotraBM17}. In particular, this restriction enables the computation of a \emph{causal discrimination score}, defined as the fraction of instances in the feature space suffering from causal discrimination, a measure which is only meaningful as long as the feature space is finite and instances therein can be exhaustively enumerated. While one can play around this limitation by using binning to discretize the feature space, the testing approach in~\cite{GalhotraBM17} is not exhaustive for scalability reasons, hence it can only identify counter-examples suffering from causal discrimination, but it cannot prove a lack of causal discrimination. Similar criticisms apply to other more recent proposals on fairness testing~\cite{AggarwalLNDS19, UdeshiAC18, BlackYF20}.

In the present work, we improve over the state of the art by proposing a new verification technique to formally verify the fairness guarantees supported by tree-based models. In particular, our technique allows one to automatically identify subsets of the feature space where lack of causal discrimination is ensured, rather than just counter-examples suffering from causal discrimination. Concretely, we first discuss how one can verify lack of causal discrimination for tree ensembles given a continuous, unbounded subset of the feature space $\feats' \subseteq \feats$ as input (Section~\ref{sec:verification}). We then build on this idea to design an effective algorithm to automatically synthesize sufficient conditions ensuring lack of causal discrimination, expressed as traditional propositional logic formulas, which are readily understandable by human experts (Section~\ref{sec:synthesis}). Our algorithm is iterative and its execution can be safely stopped before convergence to improve performance, without sacrificing soundness, i.e., only sufficient conditions are returned by the algorithm. Moreover, the algorithm is proved complete upon termination, i.e., the logical formulas returned by the algorithm may eventually provide a complete characterization of the fairness guarantees given by the classifier.

\subsection{Verification Algorithm}
\label{sec:verification}
The first problem we investigate can be formulated as follows: given a decision tree ensemble $T$, a set of sensitive features $S$ and a subset of the feature space $\feats' \subseteq \feats$, verify that $T$ does not perform causal discrimination on $\feats'$.

To answer this question, we leverage the \emph{stability} notion from the adversarial machine learning literature~\cite{RanzatoZ20}. The idea underlying stability is that classifiers deployed in adversarial settings should not be fooled by adversarial manipulations of test instances, designed to force the classifier to return wrong predictions (so-called evasion attacks). This can be enforced by requiring the classifier to be ``stable'' with respect to adversarial manipulations, i.e., to stick to its original predictions despite such manipulations. For example, malware detectors must not change their predictions as the result of semantically-preserving adversarial manipulations of malicious software, e.g., if malware replaces a given API call with an equivalent set of instructions, it should still be classified as malware. Formally, stability is defined as follows.

\begin{definition}[Stability]
Given a classifier $f: \feats \rightarrow \labels$, an instance $\vec{x} \in \feats$ and a set of adversarial manipulations $A(\vec{x})$, $f$ is \emph{stable} on $\vec{x}$ if and only if $\forall \vec{z} \in A(\vec{x}): f(\vec{z}) = f(\vec{x})$.  
\end{definition}

The idea of stability generalizes from the adversarial setting to the fairness setting, because lack of causal discrimination requires that arbitrary changes to the sensitive features $S$ must not affect the classifier predictions. A similar observation for a different definition of fairness has been performed in recent independent work~\cite{RanzatoUZ21}. The important point to note, though, is that while stability predicates on a specific instance $\vec{x} \in \feats$, lack of causal discrimination predicates over a potentially unbounded set of instances $\feats' \subseteq \feats$, hence traditional approaches to stability/fairness verification such as~\cite{RanzatoZ20} cannot be directly applied to verify lack of causal discrimination on $\feats'$.

We thus propose to leverage recent solutions for the \emph{data-independent} verification of decision tree ensembles~\cite{CalzavaraCLMO21,TornblomN20} to verify lack of causal discrimination in a continuous, unbounded subset of the feature space $\feats'$. These verification approaches are data-independent because they do not analyze the behavior of the ensemble on specific test instances, but they rather analyze the structure of the ensemble to allow the verification of properties predicating over all the instances in (a subset of) the feature space. Concretely, data-independent analyses operate in terms of a set of \emph{hyper-rectangles} $H \subseteq \feats$ such that all instances $\vec{x} \in H$ satisfy a property of interest, e.g., lead to the same prediction or enjoy stability. In particular, we build on the following definition of data-independent stability analysis for our verification purposes.

\begin{definition}[Data-Independent Stability Analysis]
\label{def:di-stability}
Let $f: \feats \rightarrow \labels$ be a classifier and let $S$ be a set of sensitive features. A \emph{data-independent stability analysis} is an algorithm which takes as input $f$ and $S$ to return as output a set of hyper-rectangles $U$ satisfying the following property: for every instance $\vec{x} \in \feats$, if there exists $\vec{z} \in \flip{\vec{x}}$ such that $f(\vec{z}) \neq f(\vec{x})$, then there exists $H \in U$ such that $\vec{x} \in H$.
\end{definition}

Intuitively, the output of a data-independent stability analysis over-approximates the subset of the feature space where the classifier violates stability for some instances therein, i.e., if $\vec{x}$ suffers from causal discrimination, then the result of the analysis must include a hyper-rectangle $H$ such that $\vec{x} \in H$. In the fairness setting, this means that the union of the hyper-rectangles returned by the analysis over-approximates the set of the counter-examples suffering from causal discrimination. Note that we only require an over-approximation rather than an exact characterization, because this is a minimal assumption for sound fairness verification and data-independent analysis might be approximated by design for scalability reasons~\cite{CalzavaraCLMO21}.

If a data-independent stability analysis is available, solving our problem of interest is straightforward. In particular, given a decision tree ensemble $T$ and a set of sensitive features $S$, we can run the data-independent stability analysis on $T$ and $S$ to produce the set of hyper-rectangles $U$. To prove that $T$ does not perform causal discrimination on $\feats' \subseteq \feats$, it suffices to show that there exists no $H \in U$ such that $H \cap \feats' \neq \emptyset$. 

We now discuss how the verification approaches in~\cite{CalzavaraCLMO21,TornblomN20} can be used to implement a data-independent stability analysis as required by our definition, thus allowing us to take advantage of existing state-of-the-art verification approaches.

\subsubsection{Resilience Verifier~\cite{CalzavaraCLMO21}}
Recent work on adversarial machine learning introduced the \emph{resilience} notion to characterize the security of classifiers deployed in adversarial settings. At the core of its resilience verification algorithm, there is a data-independent stability analysis for decision tree ensembles, which can be readily leveraged for our purposes by modeling an attacker such that, for all instances $\vec{x} \in \feats$, $A(\vec{x}) = \flip{\vec{x}}$. Intuitively, this attacker has full control of the value of the sensitive features $S$, so stability against such an attacker ensures that causal discrimination is not possible. The output of the analysis is a set of hyper-rectangles over-approximating the subset of the feature space where the decision tree ensemble is unstable, thus satisfying the conditions of Definition~\ref{def:di-stability}.

\subsubsection{VoTE Checker~\cite{TornblomN20}}
The VoTE Checker takes as input a decision tree ensemble and produces as output a partitioning of the feature space $\feats$ in terms of the equivalence classes induced by the ensemble. Each equivalence class is represented as a pair $(H,y)$, meaning that all the instances $\vec{x} \in H$ traverse the same path combination in the ensemble leading to prediction $y$. If there exist two equivalence classes $(H_1,y_1),(H_2,y_2)$ such that: $(i)$ $y_1 \neq y_2$, $(ii)$ $H_1,H_2$ have a non-empty intersection for all the features $j \not\in S$, and $(iii)$ $H_1,H_2$ have different intervals for some feature $k \in S$, then $H_1 \cup H_2$ includes a counter-example suffering from causal discrimination, because it contains two instances differing just for some sensitive feature $k$ where the classifier returns two different predictions. It is thus possible to add both $H_1$ and $H_2$ to the set of hyper-rectangles to be returned by the analysis, thus satisfying the conditions of Definition~\ref{def:di-stability}.

\subsection{Synthesis Algorithm}
\label{sec:synthesis}
We now consider a different, more interesting problem: given a decision tree ensemble $T$ and a set of sensitive features $S$, characterize the subset of the feature space $\feats' \subseteq \feats$ such that $T$ does not perform causal discrimination on $\feats'$. 

Intuitively, this problem can be conservatively solved by subtracting from $\feats$ all the hyper-rectangles $H \in U$ returned by the data-independent stability analysis assumed in the previous section, thus under-approximating the subset of the feature space where $T$ is stable. This approach would be sound, but also computationally inefficient, because subtracting hyper-rectangles suffers from an exponential blowup with respect to the dimensionality of the feature space. Indeed, given two hyper-rectangles with $d$ features, their subtraction might generate $O(d)$ hyper-rectangles in the general case, leading to $O(d^{|U|})$ hyper-rectangles in the worst case at the end of the subtraction process. This limitation can be circumvented by avoiding the computation of the subtraction and by directly reasoning in terms of instances falling out from the hyper-rectangles $U$. However, this would make it hard to characterize $\feats'$ in a human-understandable way, due to both the sheer number of the hyper-rectangles and the potentially large dimensionality of the feature space. 

We thus propose an iterative algorithm designed to incrementally generate increasingly complex sufficient conditions ensuring lack of causal discrimination, expressed as traditional logic formulas. This way, the first iterations of the algorithm can efficiently generate conditions which are short and easy to understand for human experts, while being arguably the most useful to analysts; the more analysis time and computational resources are available, the more complex conditions can be identified, thus detecting other subsets of the feature space where lack of causal discrimination is ensured. In other words, each iteration of the algorithm extends the sound approximation identified by the previous iteration by accounting for more complicated sufficient conditions for fairness, until the whole relevant subset of the feature space is covered by the conditions (or an early stopping criterion is met).

\subsubsection{Overview}
Our algorithm is inspired by the classic Apriori algorithm for frequent itemset mining~\cite{AgrawalS94}. We define an \emph{item} $i$ as a formula of the form $x_f \leq v$ or $x_f > v$, where $f \in \{1,\ldots,d\}$ identifies a feature and $v \in \R$. An \emph{itemset} $I$ is a set of formulas $\{i_1,\ldots,i_n\}$, interpreted in conjunctive form. For example, the itemset $I = \{x_1 > 1, x_1 \leq 3, x_2 > 5\}$ identifies all the instances such that the first feature is in the interval $(1,3]$ and the second feature is in the interval $(5,+\infty)$. Formally, we write $\interp{i}$ for the set of the instances identified by item $i$ and we let $\interp{I} = \bigcap_{i \in I} \interp{i}$ stand for the set of the instances identified by the itemset $I$.

An itemset identifies a subset of the feature space which is assessed for lack of causal discrimination, thus allowing us to leverage the following \emph{monotonicity} property to prune the search space: if $I$ enjoys lack of causal discrimination, then any other $I' \supseteq I$ does that as well, because $\interp{I'} \subseteq \interp{I}$; hence, $I'$ itself does not need to be analyzed. This property allows one to assess itemsets for lack of causal discrimination by starting from those with smaller cardinality.

To exemplify how the algorithm works at a high level, consider a two-dimensional feature space and assume $U$ contains just two hyper-rectangles $H_1 = \langle (1,5], (3,8] \rangle$ and $H_2 = \langle (4,7], (2,6] \rangle$. Our goal is characterizing those instances falling \emph{outside} both $H_1$ and $H_2$. Instances laying outside $H_1$ can be described as the union of the following four hyper-rectangles:
\begin{itemize}
    \item $H_{11} = \langle (-\infty,1], (-\infty,+\infty) \rangle$, represented as the itemset $I_{11} = \{x_1 \leq 1\}$
    \item $H_{12} = \langle (5,+\infty), (-\infty,+\infty) \rangle$, represented as the itemset $I_{12} = \{x_1 > 5\}$
    \item $H_{13} = \langle (-\infty,+\infty), (-\infty,3] \rangle$, represented as the itemset $I_{13} = \{x_2 \leq 3\}$
    \item $H_{14} = \langle (-\infty,+\infty), (8,+\infty) \rangle$, represented as the itemset $I_{14} = \{x_2 > 8\}$
\end{itemize}

Instances laying outside $H_2$ can instead be described as the union of the following four hyper-rectangles:
\begin{itemize}
    \item $H_{21} = \langle (-\infty,4], (-\infty,+\infty) \rangle$, represented as the itemset $I_{21} = \{x_1 \leq 4\}$
    \item $H_{22} = \langle (7,+\infty), (-\infty,+\infty) \rangle$, represented as the itemset $I_{22} = \{x_1 > 7\}$
    \item $H_{23} = \langle (-\infty,+\infty), (-\infty,2] \rangle$, represented as the itemset $I_{23} = \{x_2 \leq 2\}$
    \item $H_{24} = \langle (-\infty,+\infty), (6,+\infty) \rangle$, represented as the itemset $I_{24} = \{x_2 > 6\}$
\end{itemize}

To identify instances laying outside both $H_1$ and $H_2$, we first inspect all the itemsets above and we check whether they identify subsets of the feature space intersecting $H_1$ or $H_2$. If we do not find overlaps, the itemsets already represent instances falling out both $H_1$ and $H_2$, hence causal discrimination cannot happen there. In our example, we identify the itemsets $I_{11}$, $I_{14}$, $I_{22}$ and $I_{23}$ as sufficient conditions for lack of causal discrimination; note that the first and the third itemsets only involve the feature $x_1$, while the second and the fourth itemsets only involve the feature $x_2$. The other itemsets $I_{12}$, $I_{13}$, $I_{21}$ and $I_{24}$, instead, identify subsets of the feature space where causal discrimination might potentially happen. These itemsets are thus combined together to generate additional itemsets to check, each possibly using both features $x_1$ and $x_2$, leading to conditions of higher complexity. For example, $I_{12}$ and $I_{24}$ generate the new itemset $\{x_1 > 5, x_2 > 6\}$, which represents again instances falling out both $H_1$ and $H_2$, thus identifying sufficient conditions for lack of causal discrimination. Instead, $I_{12}$ and $I_{13}$ generate the new itemset $\{x_1 > 5, x_2 \leq 3\}$, which overlaps with $H_2$, hence it identifies a subset of the feature space where causal discrimination may happen. This itemset may thus be combined with other itemsets to undergo further refinements over $x_1$ and $x_2$, leading to smaller intervals on them, possibly identifying additional sufficient conditions for proving lack of causal discrimination.

\subsubsection{Algorithm}
Having defined the key intuitions of our proposal, we are now ready to present the details of the synthesis algorithm (Algorithm~\ref{alg:synthetize}). The algorithm takes as input a tree ensemble $T$ and a set of sensitive features $S$ to return as output a set of sufficient conditions on the feature space ensuring lack of causal discrimination. The algorithm starts by invoking the {\sc Analyze} function over $T$ and $S$, which implements a data-independent stability analysis (cf. Definition~\ref{def:di-stability}) returning a set of hyper-rectangles $U$ where $T$ may be unstable (line 2). The algorithm then initializes an empty set of fairness conditions $F$ and generates a set of candidates $C$, initialized with itemsets involving just a single feature, as described in our example; this step is implemented by the {\sc Gen\_Itemsets} function, whose straightforward implementation is omitted (lines 3-4). Each itemset $I \in C$ is checked against $U$: if $\interp{I}$ does not intersect any hyper-rectangle $H \in U$, the itemset identifies a sufficient condition for lack of causal discrimination, hence it is added to the set of fairness conditions $F$ (lines 5-7). The itemsets which do not immediately contribute to extending $F$ are instead used in the main loop of the algorithm (lines 8-19). In particular, all such itemsets are combined with each other through a \emph{meet} operator $\sqcap$ to produce new itemsets to analyze; such itemsets can either be proved fair or undergo additional refinements at later iterations, as long as there are candidates to process. The meet operator is defined and commented below.

\begin{definition}[Meet Operator]
\label{def:meet}
Given two itemsets $I_1,I_2$ such that $|I_1| = |I_2| = k$ and $|I_1 \cap I_2| = k-1$, we define their \emph{meet} $I_1 \sqcap I_2$ as the itemset $I = I_1 \cup I_2$, provided that the following conditions hold:
\begin{enumerate}
    \item $\interp{I} \neq \emptyset$, i.e., the itemset $I$ identifies a non-empty subset of instances;
    \item $\interp{I} \subset \interp{I_1}$ and $\interp{I} \subset \interp{I_2}$, i.e., the itemset $I$ identifies less instances than both $I_1$ and $I_2$.
\end{enumerate}

If any of the aforementioned conditions does not hold, the meet $I$ does not exist.
\end{definition}

Observe that, given two itemsets $I_1,I_2$ of cardinality $k$ sharing $k-1$ elements, their meet $I_1 \sqcap I_2$ produces a new itemset $I_1 \cup I_2$ of cardinality $k+1$, i.e., itemsets are generated in increasing order of cardinality to leverage the discussed monotonicity property. The two technical conditions of Definition~\ref{def:meet} just ensure that testing the new generated itemset might be useful, i.e., the new itemset is non-empty and differs from the previously generated itemsets $I_1,I_2$. Although these conditions are formulated in a declarative style to simplify their understanding, the meet operator $\sqcap$ is straightforward to implement in practice. 
\writtenbyLC{In particular, let $i^*$ be the (only) item such that $i^* \in I_2 \setminus I_1$ and let $f^*$ be the feature predicated upon by $i^*$, then we let $I_1 \sqcap I_2 = I_1 \cup \{i^*\}$ provided that the two conditions of the definition are satisfied. The implementation of the first condition checks whether all the items $i \in I_1$ predicating on $f^*$ have a non-empty intersection with $i^*$. The implementation of the second condition, instead, amounts to verifying that adding $i^*$ to $I_1$ identifies a smaller interval for the feature $f^*$.}

There is just one point left to discuss, which is the key observation that not all the itemsets must undergo the potentially expensive intersection test against $U$. In particular, if $\interp{I} \subseteq \interp{I'}$ for some $I'$ which we already proved fair, $I$ can be ignored, because it does not identify new sufficient conditions for fairness. The check at line 14 implements this optimization. Note that it is easy to move from the declarative style of this check to its implementation\writtenbyLC{, because checking $\interp{I} \subseteq \interp{I'}$ amounts to checking that, for all features $f$, the interval on $f$ identified by $I$ is included in the interval on $f$ identified by $I'$.}

\begin{algorithm}[t]
\caption{Synthesizing Fairness Conditions}
\label{alg:synthetize}
\begin{algorithmic}[1]
\Function{Synthesize}{$T,S$}
    \State $U \gets \Call{Analyze}{T,S}$
    \State $F \gets \emptyset$
    \State $C \gets \Call{Gen\_Itemsets}{U}$
    \For{$I \in C$}
        \If{$\forall H \in U: \interp{I} \cap H = \emptyset$}
            \State{$F \gets F \cup \{I\}$}
        \EndIf
    \EndFor
    \State $C \gets C \setminus F$
    \While{$C \neq \emptyset$}
        \State{$C' \gets \emptyset$}
        \For{$I_1 \in C$}
            \For{$I_2 \in C$}
                \State{$I \gets I_1 \sqcap I_2$}
                \If{$I \text{ exists } \wedge \not \exists I' \in F: \interp{I} \subseteq \interp{I'}$}
                    \If{$\forall H \in U: \interp{I} \cap H = \emptyset$}
                        \State{$F \gets F \cup \{I\}$}
                    \Else
                        \State{$C' \gets C' \cup \{I\}$}
                    \EndIf
                \EndIf
            \EndFor
        \EndFor
        \State{$C \gets C'$}
    \EndWhile
    \State \Return{$F$}
\EndFunction
\end{algorithmic}
\end{algorithm}

\subsubsection{Formal Results}
We can prove that our algorithm is both sound and complete, as formalized below. Soundness ensures that any itemset $I$ returned by the synthesis algorithm is a sufficient condition for fairness, i.e., instances in $\interp{I}$ cannot suffer from causal discrimination.

\begin{theorem}[Soundness]
For any decision tree ensemble $T$ and set of sensitive features $S$, the call $\Call{Synthesize}{T,S}$ returns a set of itemsets $F$ such that, for every $I \in F$, $T$ does not perform causal discrimination on $\interp{I}$.
\end{theorem}
\begin{proof}
The call $\Call{Analyze}{T,S}$ returns a set of hyper-rectangles $U$ satisfying the following property: for every instance $\vec{x} \in \feats$, if there exists $\vec{z} \in \flip{\vec{x}}$ such that $f(\vec{z}) \neq f(\vec{x})$, then there exists $H \in U$ such that $\vec{x} \in H$. This means that $T$ cannot perform causal discrimination on any $\feats' \subseteq \feats$ such that $\forall H \in U: \feats' \cap H = \emptyset$. The conclusion follows by observing that any itemset $I$ which is added to $F$ must satisfy this property (cf. lines 6-7 and lines 15-16).
\end{proof}

Completeness, instead, ensures that the combination of all the itemsets $F$ returned by the synthesis algorithm coincides with the subset of the feature space disjoint from the result of the data-independent stability analysis, i.e., it represents the ideal outcome of the synthesis algorithm. Note that, if the data-independent stability analysis is not over-approximated but exact, this ensures that $F$ covers all the instances where the classifier is fair.

\begin{theorem}[Completeness]
For any decision tree ensemble $T$ and set of sensitive features $S$, the call $\Call{Synthesize}{T,S}$ returns a set of itemsets $F$ such that $\bigcup_{I \in F} \interp{I} = \feats \setminus \bigcup_{H \in U} H$, where $U$ is the output of $\Call{Analyze}{T,S}$.
\end{theorem}
\begin{proof}
We prove the equality of the two sets by showing that the first is included in the second and vice-versa. Specifically:
\begin{itemize}
    \item Consider any instance $\vec{x} \in \bigcup_{I \in F} \interp{I}$, we show that for all $H \in U$ we have $\vec{x} \not\in H$, which implies $\vec{x} \in \feats \setminus \bigcup_{H \in U} H$. Indeed, the algorithm ensures that for all $I \in F$ we have that $\forall H \in U: \interp{I} \cap H = \emptyset$ (cf. lines 6-7 and lines 15-16).
    
    \item Consider any instance $\vec{x} \in \feats \setminus \bigcup_{H \in U} H$, we show that there exists $I \in F$ such that $\vec{x} \in \interp{I}$, which implies $\vec{x} \in \bigcup_{I \in F} \interp{I}$. We first observe that, for all $H_j \in U$, the call $\Call{Gen\_Itemsets}{U}$ returns a set of at most $2d$ itemsets: 
    \[
    C_j = \{\{i_1\},\ldots,\{i_{2d}\}\},
    \]
    such that $\vec{x} \in \interp{i_k}$ for some $i_k$; we refer to such item $i_k$ as the \emph{witness} for $H_j$. The itemset $I$ includes all the witnesses for each $H_j \in U$, thus ensures that $\vec{x} \in \interp{I} = \bigcap_{i \in I} \interp{i}$. The conclusion follows by observing that either the itemset $I$ or another itemset $I'$ such that $\interp{I} \subseteq \interp{I'}$ is eventually enumerated by the algorithm.
\end{itemize}
\end{proof}

\subsubsection{Implementation}
We implemented the synthesis algorithm presented in this section in C++. Our implementation leverages the data-independent stability analysis used by the resilience verifier presented in~\cite{CalzavaraCLMO21}, which we simply leverage as a black box. Although the implementation is a rather direct translation of the pseudocode in Algorithm~\ref{alg:synthetize}, a few important details are worth discussing. A first point to note is that our implementation supports a user-specified early stopping criterion in terms of a maximum number of iterations of the algorithm. This is useful because, like Apriori, our algorithm has an exponential time complexity with respect to the number of items in the worst case~\cite{AgrawalS94}. However, we empirically observed in our experimental evaluation that a small number of short conditions already allow one to largely characterize the fairness guarantees of tree-based classifiers, hence restricting the number of iterations saves analysis time, while leading to just a small loss in precision in practice. Note that early stopping preserves the soundness of the analysis, while obviously sacrificing its completeness.

Another important aspect of the implementation is the treatment of \emph{categorical} features, i.e., features which do not take arbitrary values in $\R$, but may only take values from a set of unordered options, e.g., gender or ethnicity. Categorical features are not part of our model for simplicity, however, they are not difficult to handle in practice. \writtenbyLC{In particular, we handle binary features just like numerical features, since the threshold 0.5 can be used to tell apart the two possible values of the feature. As for other categorical features, we deal with them by using one-hot-encoding, i.e., we generate a new feature $f$ for each possible value of the original feature, with the idea that just one such feature will have a value greater than 0.5. We then enforce that an itemset cannot contain two or more items of the form $x_f > 0.5$ that predicate on features resulting by one-hot-encoding the same categorical feature. Since the categorical features and the result of their one-hot-encoding are known, the implementation of the meet operator checks whether the generated itemsets respect the described integrity condition and the meet is considered undefined in case of violations. The resulting itemsets in the result $F$ are post-processed in order to improve their readability. In particular, we collapse itemsets that differ just for formulas predicating on different values of the same categorical feature by merging these formulas in a single formula using a disjunction.}

\writtenbyLC{Finally, we mention a few optimizations of the implementation that we borrowed and adapted from Apriori~\cite{AgrawalS94}. The items within an itemset and the itemsets themselves are kept ordered to reduce the number of executions of the meet operator at line 13 and optimize its checks. The items within an itemset are ordered as follows: first, we follow the lexicographic order of the feature involved in the item; then, for any two items involving the same feature, the formulas with predicate $\leq$ precede the ones with predicate $>$; finally, the items with predicate $\leq$ are ordered by decreasing value of the threshold, while we do the opposite for the other items. The order of the items makes it possible to define the prefix of an itemset of size $k$ as the ordered sequence of its first $k-1$ items. Itemsets are ordered to examine itemsets sharing the same prefix consecutively, since only itemsets sharing the same prefix satisfy the preconditions for performing the meet operation. This allows one to reduce the number of times that the meet operation is attempted between two itemsets that do not satisfy the preconditions of the meet definition. In particular, two ordered itemsets of size $k$ are lexicographically ordered using the established ordering of the items: at first, their prefixes are compared starting from the first item to the $(k-1)$-th item; then, if the two itemsets share the same prefix, they are ordered by comparing their last, different item. Moreover, the order of the items followed by the itemsets sharing the same prefix enables an optimization of the checks of the meet operation. Indeed, for any itemset $I'$ listed after an itemset $I$ sharing the same prefix we have $\interp{I} \not\subset \interp{I'}$. As a result, the implementation of the second condition of the meet $I \sqcap I'$ just needs to check $\interp{I'} \subset \interp{I}$. Another optimization involves the condition at line 15, which would require scanning the entire set $U$ for every generated itemset. We optimize this step by assigning to each $H \in U$ an identifier $\textit{id}_H \in \mathbb{N}$ and associating to each itemset $I \in C$ a set $\text{ids}_I = \{\text{id}_H : H \in U \wedge H \cap \interp{I} = \emptyset \}$. Then, the set of the identifiers associated to the meet $I = I_1 \sqcap I_2$ of two itemset $I_1$ and $I_2$ is $\text{ids}_I = \text{ids}_{I_1}  \cup \text{ids}_{I_2}$, since $I$ identifies less instances than both $I_1$ and $I_2$ by definition. This information is useful, because the condition at line 15 amounts just to checking whether $|\text{ids}_I| = |U|$, rather than scanning $U$.}

\section{Experimental Evaluation}
\label{sec:experiments}
In this section, we experimentally assess the effectiveness of our global fairness verification approach along different axes.

\subsection{Methodology}
    We evaluate our proposal on decision tree ensembles trained over three public datasets \writtenbyLC{used in the fairness literature~\cite{RanzatoUZ21}, associated with a binary classification task: Adult\footnote{https://archive.ics.uci.edu/ml/datasets/adult}, German\footnote{https://archive.ics.uci.edu/ml/datasets/statlog+(german+credit+data)} and Health\footnote{https://www.kaggle.com/c/hhp}. We use the attribute \textit{sex} as the binary sensitive feature. We pre-process the datasets by following three steps: (i) we normalize numerical features in the interval $[0, 1]$; (ii) we perform the one-hot-encoding of all the categorical features; (iii) we remove the features and instances containing missing values. The characteristics of the datasets are reported in Table~\ref{tab:datasets}, along with the number of categorical features before and after step (ii).}

\begin{table}[]
    \centering\caption{Dataset statistics (we report in parentheses the number of categorical features after one-hot-encoding)}
    \begin{tabular}{c|c|c|c|c}
    \textbf{Dataset} & \textbf{\#Num features} & \textbf{\#Cat features} & \textbf{\#Instances} & \textbf{\%Positive} \\
    \hline
    Adult & 6 & 7 (81) & 45,222 & 25\%\\
    German & 7 & 13 (49) & 1,000 & 70\%\\
    Health & 93 & 2 (10) & 166,842 & 68\%\\
    \end{tabular}    \label{tab:datasets}
\end{table}

We partition datasets into a training set $\dtrain$ and a test set $\dtest$, using 80-20 stratified random sampling. For each dataset, we use $\dtrain$ to train standard Random Forest (RF) models as available in sklearn~\cite{scikit-learn}. Experiments are then conducted on both $\dtest$ and a synthetic dataset $\drand$ including 100k random instances. The rationale for having these two sets is that $\dtest$ is supposed to follow the same distribution of $\dtrain$, hence it represents expected inputs at test time, while $\drand$ provides a much larger view of the entire feature space, which is interesting because our fairness guarantees are global and generalize beyond the test set. Both $\dtest$ and $\drand$ are thus useful to investigate different aspects of our proposal. We write $\dany$ in formulas to stand for any between $\dtest$ and $\drand$, using the notation $\vec{x} \in \dany$ to identify instances in $\dany$ while ignoring their class label when $\dany = \dtest$.

The experiments are designed to answer three key research questions:
\begin{enumerate}
    \item What is the \emph{precision} of the analysis performed by the synthesis algorithm? In other words, do the synthesized fairness conditions cover well the portions of the feature space where lack of causal discrimination is ensured?
    
    \item What is the \emph{explainability} of the results returned by the synthesis algorithm? Can we effectively characterize the fairness guarantees of classifiers in terms of a small number of conditions of limited complexity?
    
    \item What is the \emph{performance} of the synthesis algorithm? How is the analysis running time influenced by the size and complexity of the classifiers?
\end{enumerate}

\subsection{Precision of the Analysis}
In our first experiment, we estimate the precision of our verification approach by comparing the set of hyper-rectangles $U$ returned by the data-independent stability analysis with the formulas $F$ returned by the synthesis algorithm. The completeness theorem ensures that $F$ identifies the subset of the feature space disjoint from all the hyper-rectangles in $U$, provided that the synthesis algorithm runs up to completion. In practice, however, we expect the synthesis algorithm to be normally subject to early stopping, due to the exponential blow-up in the number of candidates to be analyzed at the different iterations and the increasing complexity of the synthesized formulas. We thus estimate the precision of a partial output obtained when the synthesis algorithm is forcefully stopped after \empirical{six} iterations, because we are interested in formulas with a small number of items, amenable for human understanding. Later experiments assess the impact of the number of iterations on the analysis results and running times.

We first compute the casual discrimination score on the dataset $\dany$ based on the set of hyper-rectangles $U$. In particular, we observe that an instance $\vec{x}$ can only suffer from causal discrimination when it belongs to some $H \in U$, so the casual discrimination score on $\dany$ can be conservatively approximated as follows:
\[
d(U,\dany) = \dfrac{|\{\vec{x} \in \dany ~|~ \exists H \in U: \vec{x} \in H\}|}{|\dany|}
\]

Similarly, we can use the formulas in $F$ to reason about fairness by observing that any instance satisfying some formula in $F$ cannot suffer from causal discrimination. This allows us to compute the following over-approximation of $d$:
\[
\ccd(F,\dany) = 1 - \dfrac{|\{\vec{x} \in \dany ~|~ \exists I \in F: \vec{x} \in \interp{I}\}|}{|\dany|}
\]

Observe that $\ccd$ is always an upper bound of $d$, because our analysis is sound. Table~\ref{tab:results} reports the values of $d$ and $\ccd$ computed for different datasets and models. We observe that $d$ and $\ccd$ coincide \writtenbyLC{for the very large majority of the} cases, meaning that the formulas in $F$ accurately characterize the subset of the feature space which is disjoint from all the hyper-rectangles in $U$, even when enforcing early stopping after \empirical{six} iterations. \writtenbyLC{Note that $\ccd$ does not coincide with $d$ just in three cases, all associated with the German dataset, but we experimentally assessed that the two scores coincide also there when increasing the number of iterations of the synthesizer to seven.} The table also reports the accuracy $a$ on the test set $\dtest$ to show that all the analyzed models perform reasonably well in practice.

\begin{table}[t]
    \centering\caption{Computed measures for different datasets and models}
    \begin{tabular}{c|c|c||c|c|c| |c|c}
    \multicolumn{3}{c}{} & \multicolumn{3}{c}{$\dtest$} & \multicolumn{2}{c}{$\drand$} \\
    Dataset & \#Trees & Depth & $a$ & $d$ & $\ccd$ & $d$ & $\ccd$ \\ 
    \hline
    \multirow{6}{*}{Adult}
      & 5 & 5 & 0.811 & 0.003 & 0.003 & 0.348 & 0.348\\ 
    & 9 & 5 & 0.826 & 0.003 & 0.003 & 0.323 & 0.323 \\ 
    & 13 & 5 & 0.824 & 0.014 & 0.014 & 0.330 & 0.330\\ 
    \cline{2-8}
    & 5 & 6 & 0.833 & 0.004 & 0.004 & 0.113 & 0.113\\ 
    & 9 & 6 & 0.834 & 0.004 & 0.004 & 0.206 & 0.206\\ 
    & 13 & 6 & 0.840 & 0.006 & 0.006 & 0.235  & 0.235\\ 
    \hline
    \multirow{6}{*}{German}
      & 5 & 5 & 0.720 & 0.230 & 0.230 & 0.349 & 0.349 \\ 
    & 9 & 5 & 0.725 & 0.230 & 0.230 & 0.349 & 0.349 \\ 
    & 13 & 5 & 0.725 & 0.240 & 0.240 & 0.373 & 0.373\\ 
    \cline{2-8}
    & 5 & 6 & 0.745 & 0.230 & 0.230 & 0.371 & 0.371\\ 
    & 9 & 6 & 0.745 & 0.325 & 0.325 &  0.437 & 0.450  \\ 
    & 13 & 6 & 0.730 & 0.375 & 0.415 & 0.486 & 0.571\\ 
    \hline
    \multirow{6}{*}{Health}
  & 5 & 5 & 0.801 & 0.008 & 0.008 & 0.027 & 0.027\\ 
    & 9 & 5 & 0.802 & 0.008 & 0.008 & 0.027 & 0.027\\ 
    & 13 & 5 & 0.803 & 0.016 & 0.016 & 0.028 & 0.028 \\ 
    \cline{2-8}
    & 5 & 6 & 0.810 & 0.018 & 0.018 & 0.209 & 0.209 \\ 
    & 9 & 6 & 0.809 & 0.018 & 0.018 & 0.209 & 0.209\\ 
    & 13 & 6 & 0.810 & 0.020 & 0.020 & 0.259 & 0.259\\ 
    \hline
    \end{tabular}
    \label{tab:results}
\end{table}

\subsection{Explainability of the Results}
To assess to what extent the formulas $F$ are explainable, we carry out \writtenbyLC{selected experiments on the most complex models that we trained on the considered datasets, i.e., ensembles of 13 trees with maximum depth six, because they are expected to be the most challenging to explain.} In the first experiment, we assess how the percentage of instances for which $F$ is able to provide a proof of fairness grows when varying the number of iterations of the synthesis algorithm. Of course, we compute this percentage with respect to the number of instances which do not suffer from causal discrimination, i.e., which might actually admit a proof of fairness. Since a run with $k$ iterations can only produce formulas involving at most $k$ items, this experiment provides insights on how complex sufficient conditions for fairness turn out to be in practice. \writtenbyLC{Figure~\ref{fig: coverage} plots the observed trend. The experimental results for the Adult dataset show that just \empirical{four} iterations of the algorithm suffice to establish fairness proofs for more than \empirical{90\%} of the instances of both $\dtest$ and $\drand$, while just \empirical{two} iterations are enough to cover basically all the instances of the two sets for the Health dataset. The German dataset is the most challenging, since \empirical{five} iterations of the algorithm are needed to cover around \empirical{80\%} of the instances in the two sets. In the end, the experiment shows that short logical formulas including at most \empirical{five} items are expressive enough to establish useful fairness proofs in practice, while being small enough to be easily understandable by human experts.}

\begin{figure*}[t]
  \centering
    \subfloat[]{\includegraphics[width=.33\textwidth]{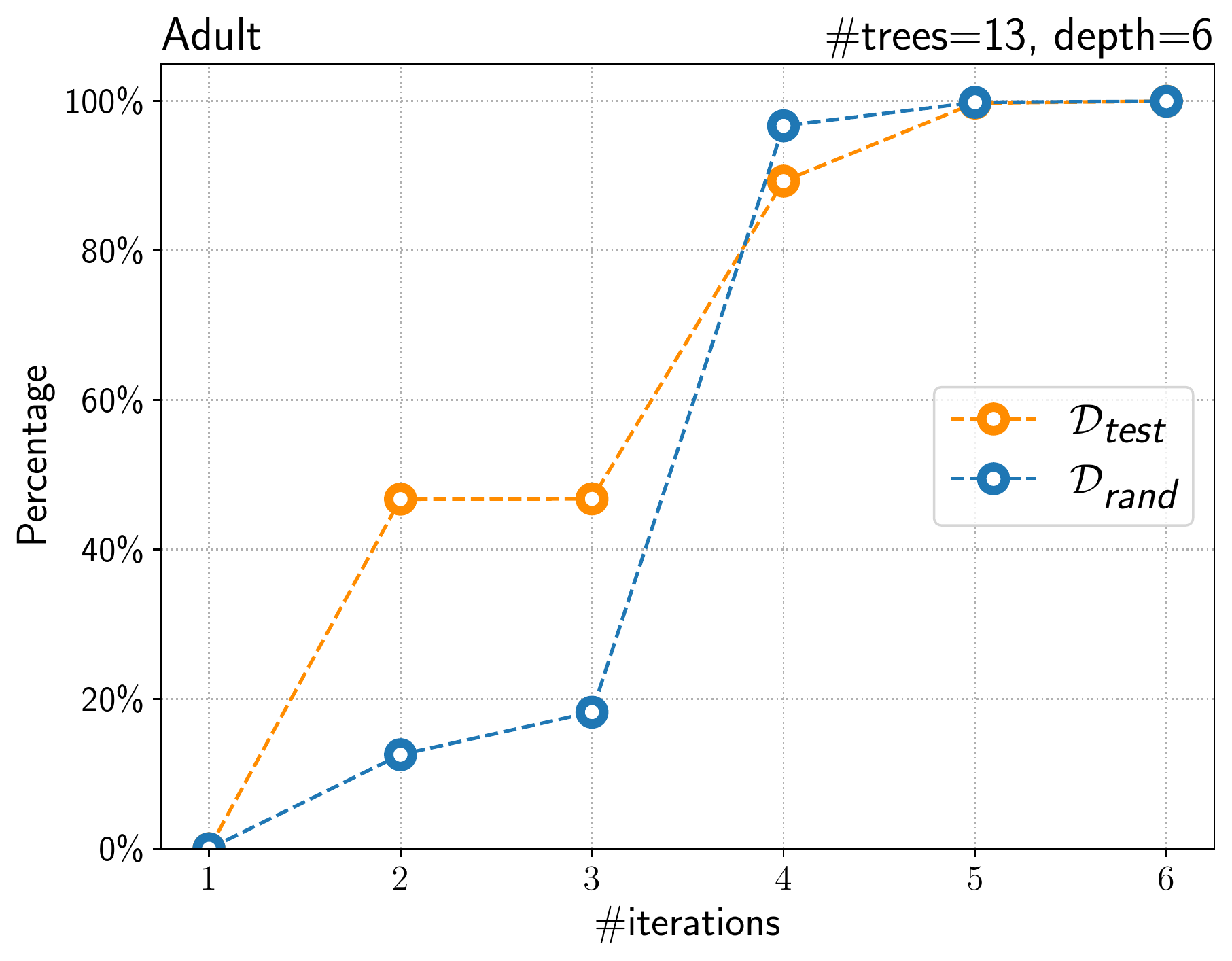}}
  \hfill
    \subfloat[]{\includegraphics[width=.33\textwidth]{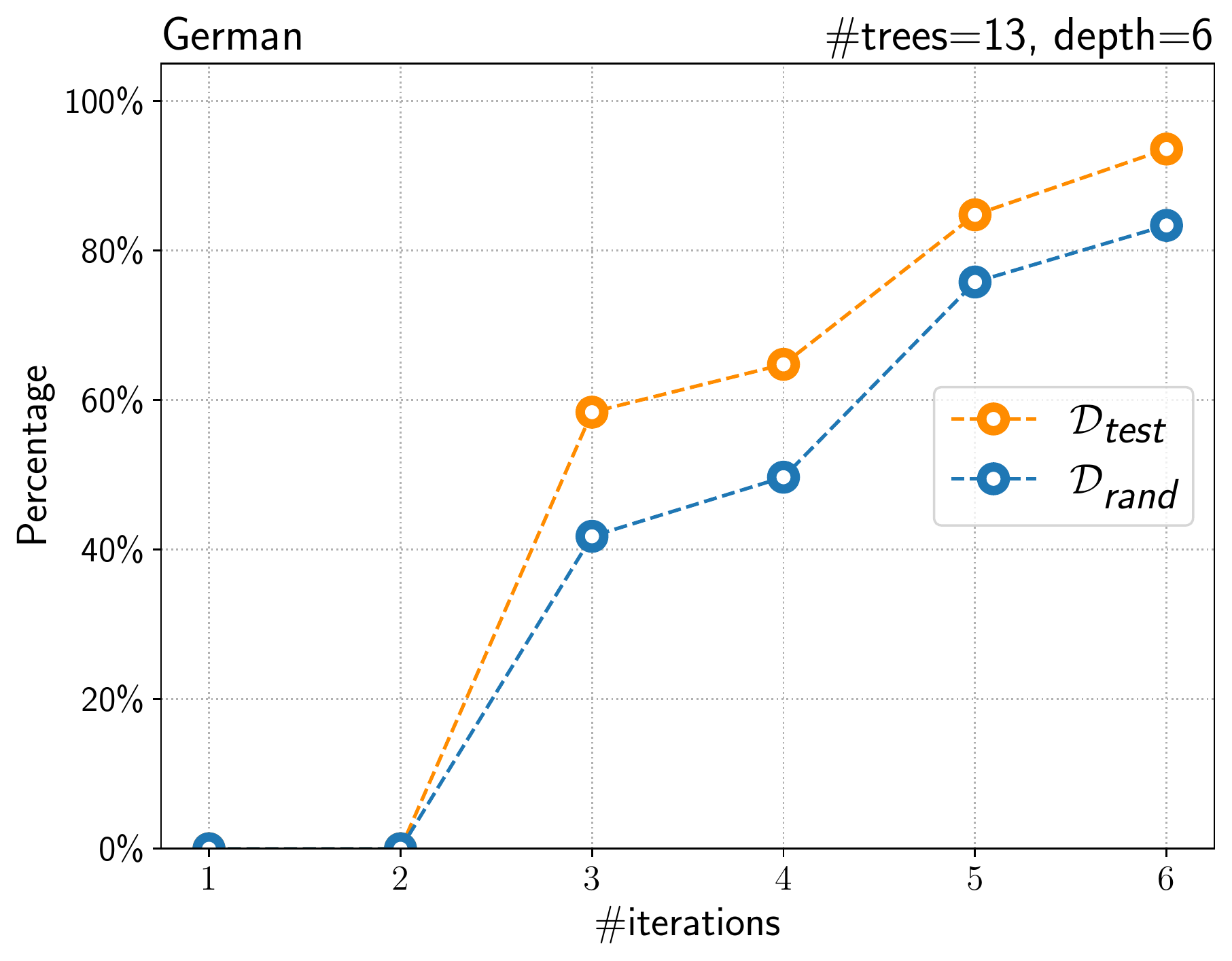}}
  \hfill
    \subfloat[]{\includegraphics[width=.33\textwidth]{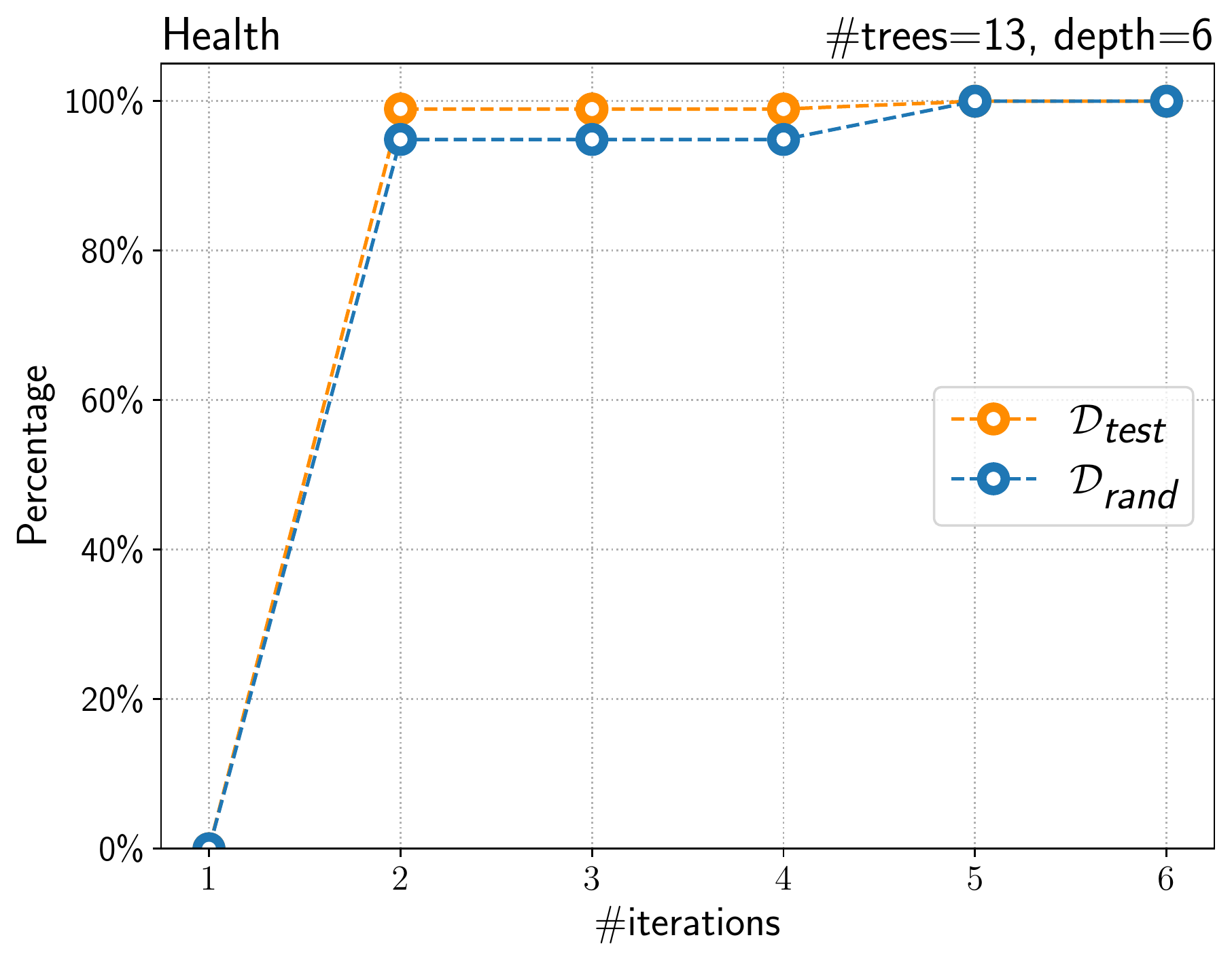}}
  \caption{Percentage of instances for which $F$ is able to provide a proof
of fairness.}
  \label{fig: coverage}
\end{figure*}

Clearly, however, our first experiment provides just a partial picture of explainability, because it captures information about the complexity of formulas, but it does not tell how many formulas should be taken into consideration by human analysts to draw useful conclusions. \writtenbyLC{Indeed, we observe that the amount of formulas can significantly grow as the number of iterations of the synthesis algorithm increases. Figure~\ref{fig: number-conds} plots how the number of synthesized formulas grows when varying the number of iterations of the synthesis algorithm, showing an exponential trend for the Adult and German datasets. We see an interesting trend in the results for the Health dataset: the number of formulas does not necessarily increase from one iteration to another, since in some cases the synthesizer is not able to produce longer formulas that provide proof of fairness. However, the figure shows a significant increase in the number of formulas when five iterations are performed, revealing that later iterations could mine longer formulas even when, during an iteration, new formulas are not discovered.}

\begin{figure*}[t]
  \centering
    \subfloat[]{\includegraphics[width=.33\textwidth]{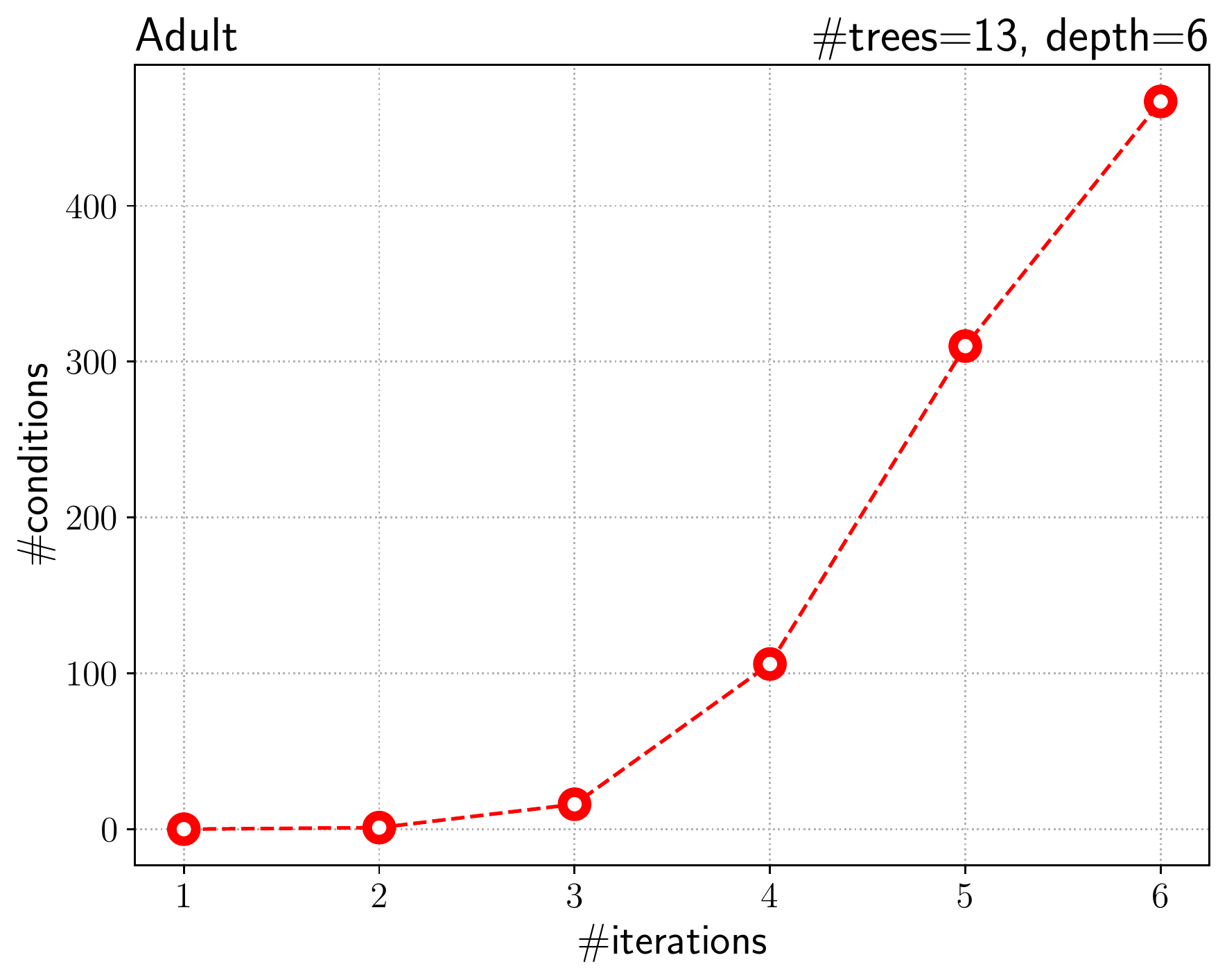}}
  \hfill
    \subfloat[]{\includegraphics[width=.33\textwidth]{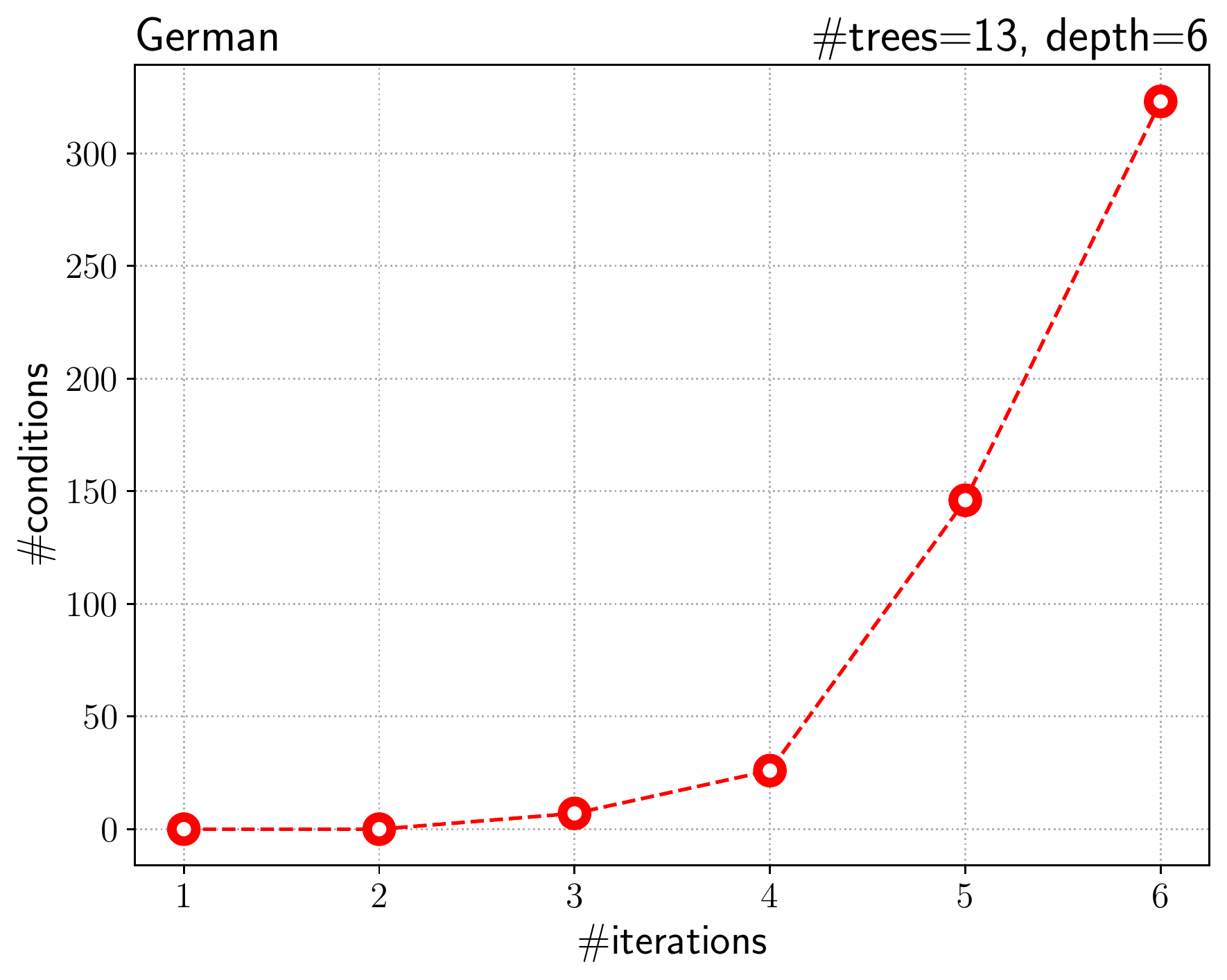}}
  \hfill
    \subfloat[]{\includegraphics[width=.33\textwidth]{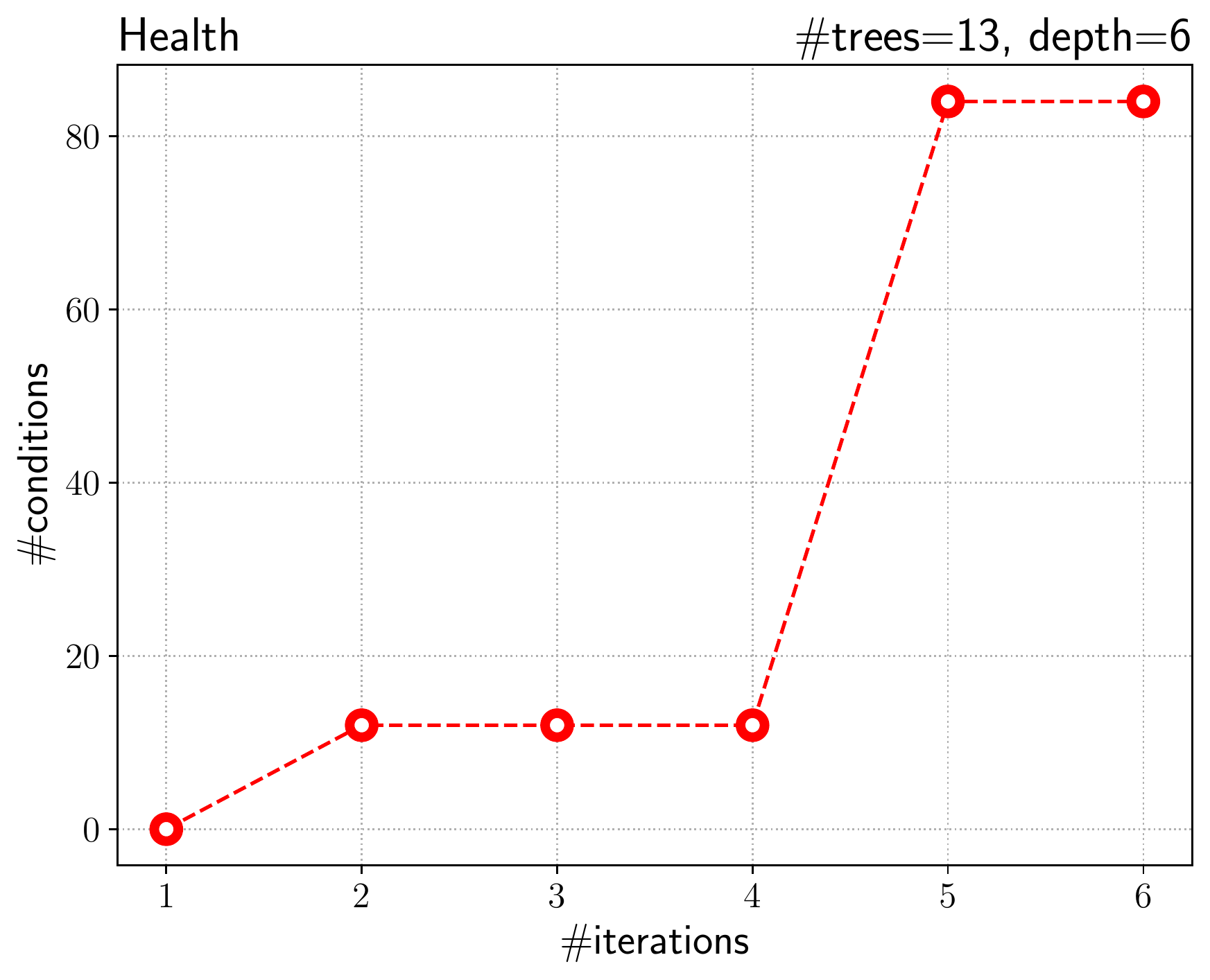}}
  \caption{Number of formulas in $F$ when varying the number of iterations.}
  \label{fig: number-conds}
\end{figure*}

Nevertheless, we can show that the number of \emph{important} formulas for human analysts is relatively small in practice. \writtenbyLC{We estimate the importance of a formula by counting the number of instances in $\dtrain$ which are covered by the formula: the intuition is that the more instances are covered by the formula, the more the formula is expressive to prove fairness according to the training data. We identify the set of the top $k$ most important formulas by means of a greedy strategy. We first select the most important formula in terms of number of covered instances, we then remove the covered instances from $\dtrain$ before selecting the second most important formula and so on, until $k$ formulas have been selected. 
Thus, by fixing the number of iterations of the synthesis algorithm and varying the number $k$, we can assess how the percentage of instances proved fair by the top $k$ formulas grows.
Figure~\ref{fig: cover-top} plots the observed trend for six iterations of the synthesis algorithm. The figure shows that for Adult and Health just the top $10$ formulas suffice to cover around $90\%$ of the $\dtest$ instances, while for German the top $20$ formulas allow one to establish a proof of fairness for $80\%$ of the instances in $\dtest$. This shows that a small number of formulas is sufficient to characterize the fairness guarantees on the test data for all the considered datasets. In general, more formulas are needed to cover synthetic instances in $\drand$. Indeed, the figure shows that, while the top $20$ formulas still allow one to cover more than $90\%$ of the instances in $\drand$ for Health, they cover just around $30\%$ of the instances of Adult and around $60\%$ of the instances of German. The difference between the results on $\dtest$ and $\drand$ can be explained by observing that the synthesized conditions depend on the thresholds learned from $\dtrain$, which is the same set used to rank the conditions, hence the top conditions generalize better on a test set with the same distribution of $\dtrain$ than on a random set. The good news is that a small number of conditions still suffice to cover a non-negligible share of $\drand$ in all cases.}

\begin{figure*}[t]
\centering
    \subfloat[]{\includegraphics[width=.33\textwidth]{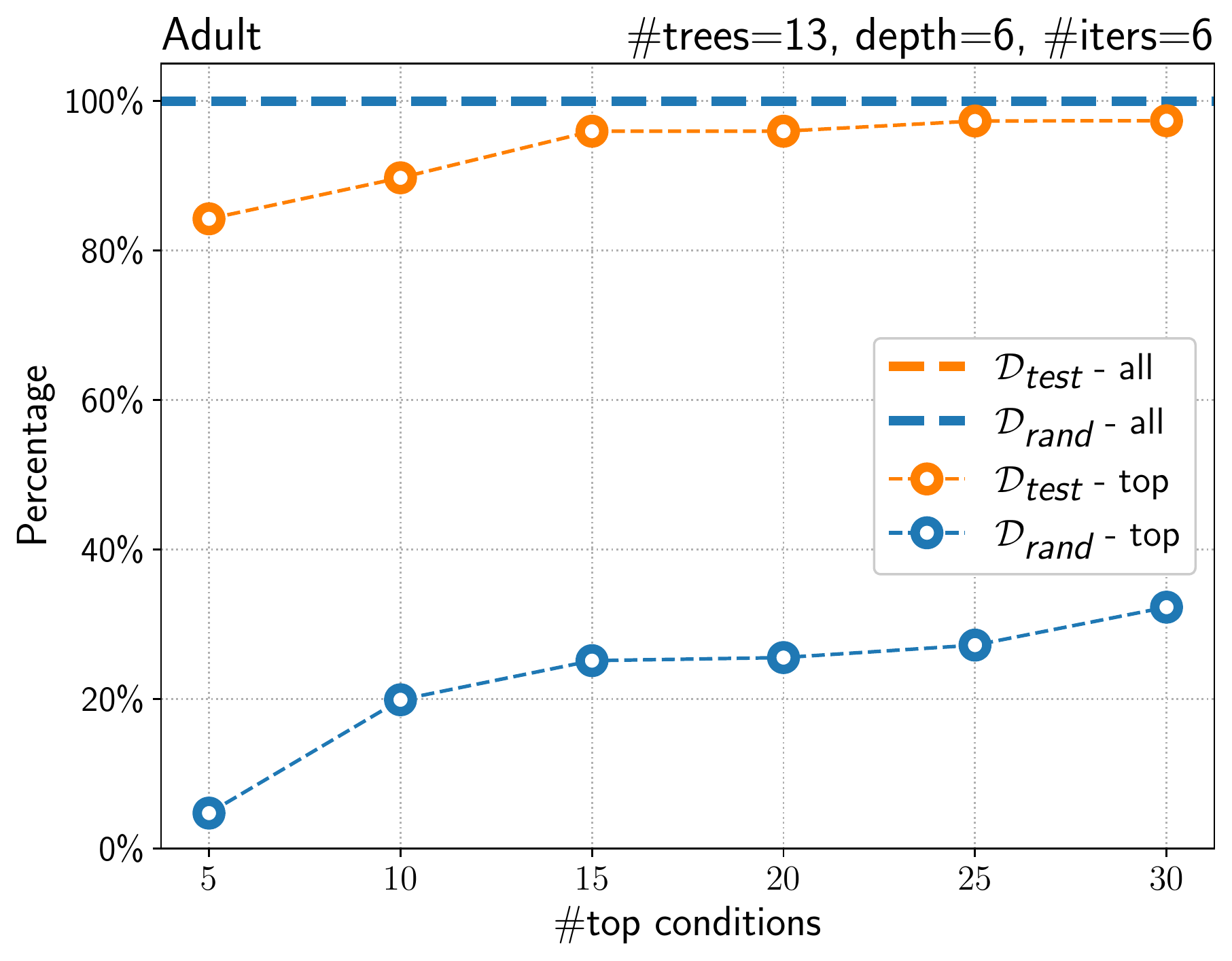}}
    \hfill
    \subfloat[]{\includegraphics[width=.33\textwidth]{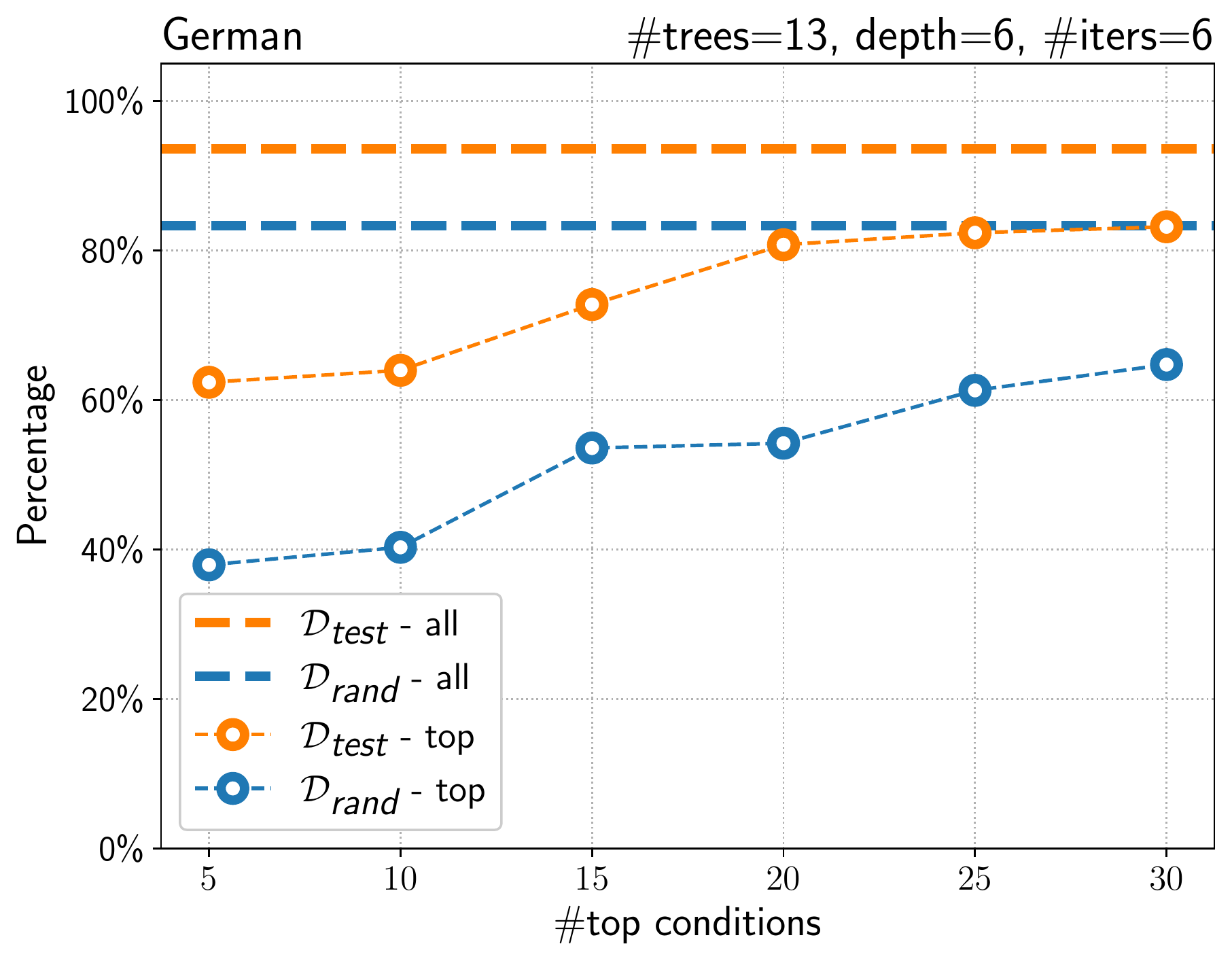}}
    \hfill
    \subfloat[]{\includegraphics[width=.33\textwidth]{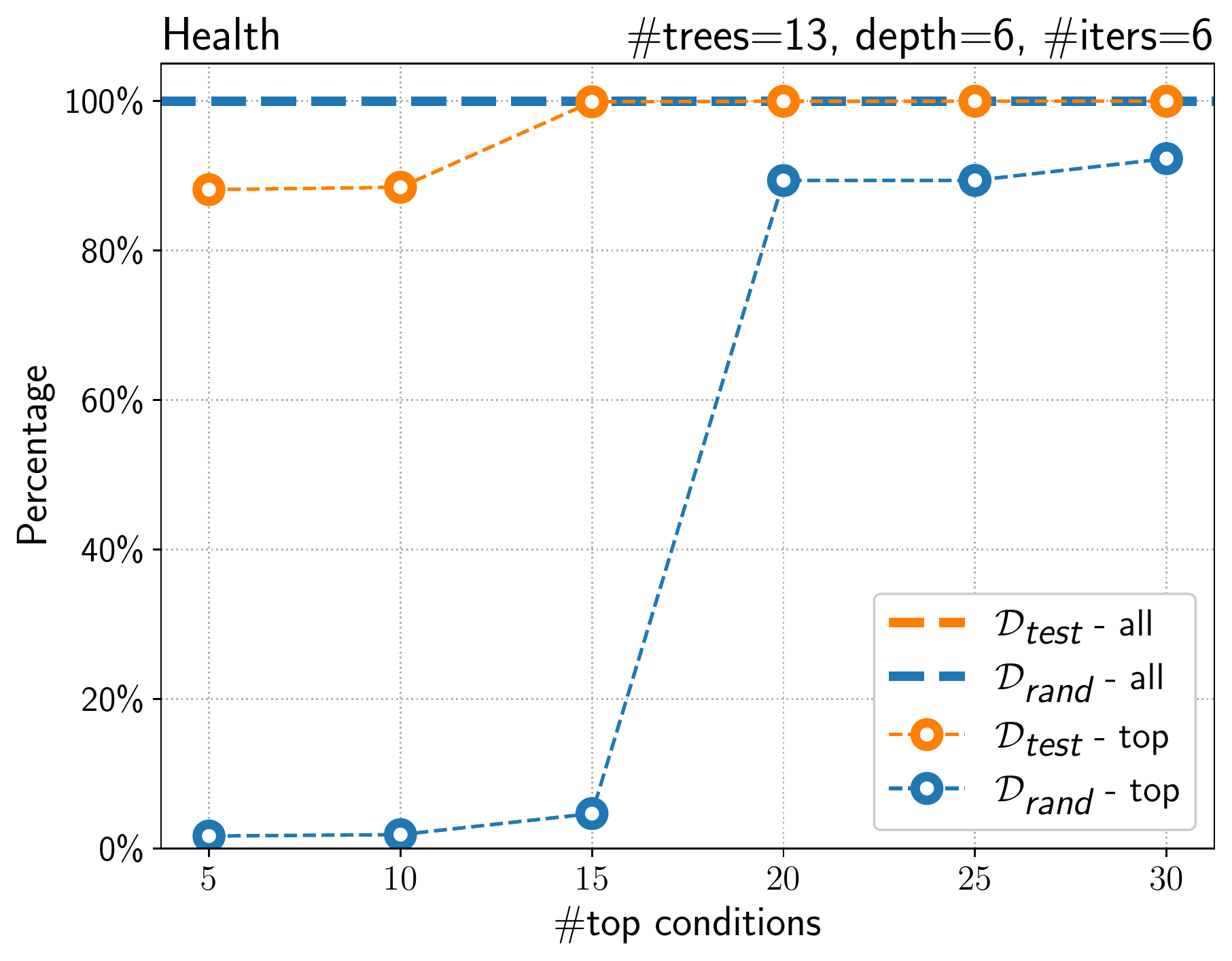}}
  \caption{Percentage of instances for which the top $k$ formulas of $F$ are able to provide a proof
of fairness.}
  \label{fig: cover-top}
\end{figure*}

We finally perform a more qualitative evaluation of the explainability of the synthesis algorithm by considering the German dataset as a case study. For space reasons, this is discussed in Appendix~\ref{sec:appendix}.

\subsection{Performance Evaluation}
We finally analyze the performance of the synthesis algorithm when varying different dimensions of our experimental setting. \writtenbyLC{Times are computed for our sequential implementation of the synthesis algorithm, running on a virtual machine with 20 cores, 98.8GB of RAM and Ubuntu 20.04.4 LTS on a server with an Intel Xeon Gold 6148 2.40GHz.} 

\writtenbyLC{First, we plot how the analysis times change when increasing the number of analysis iterations from one to seven. For the sake of readability, we only focus on ensembles of 13 trees with maximum depth six, which are the most complex models in our experimental evaluation and likely the most challenging models to analyze. The results are shown in Figure~\ref{fig: time}. The first observation is that all the models can be analyzed in a matter of minutes when performing six iterations of the algorithm, like in our previous experiments: the analysis takes around \empirical{30} minutes on the Adult dataset, \empirical{five} minutes on the German dataset and just \empirical{13} seconds on the Health dataset. The gap in the third dataset can be explained by the smaller number of categorical features therein. In general, the plots show an exponential growth of the analysis time as the number of iterations increases. Luckily, previous results show that conditions containing at most five items (i.e., generated after at most five iterations) are already enough to cover a large part of the feature space for all the analyzed models, hence further increasing the number of analysis iterations for them is unimportant in practice.}

\begin{figure*}[t]
  \centering
    \subfloat[]{\includegraphics[width=.33\textwidth]{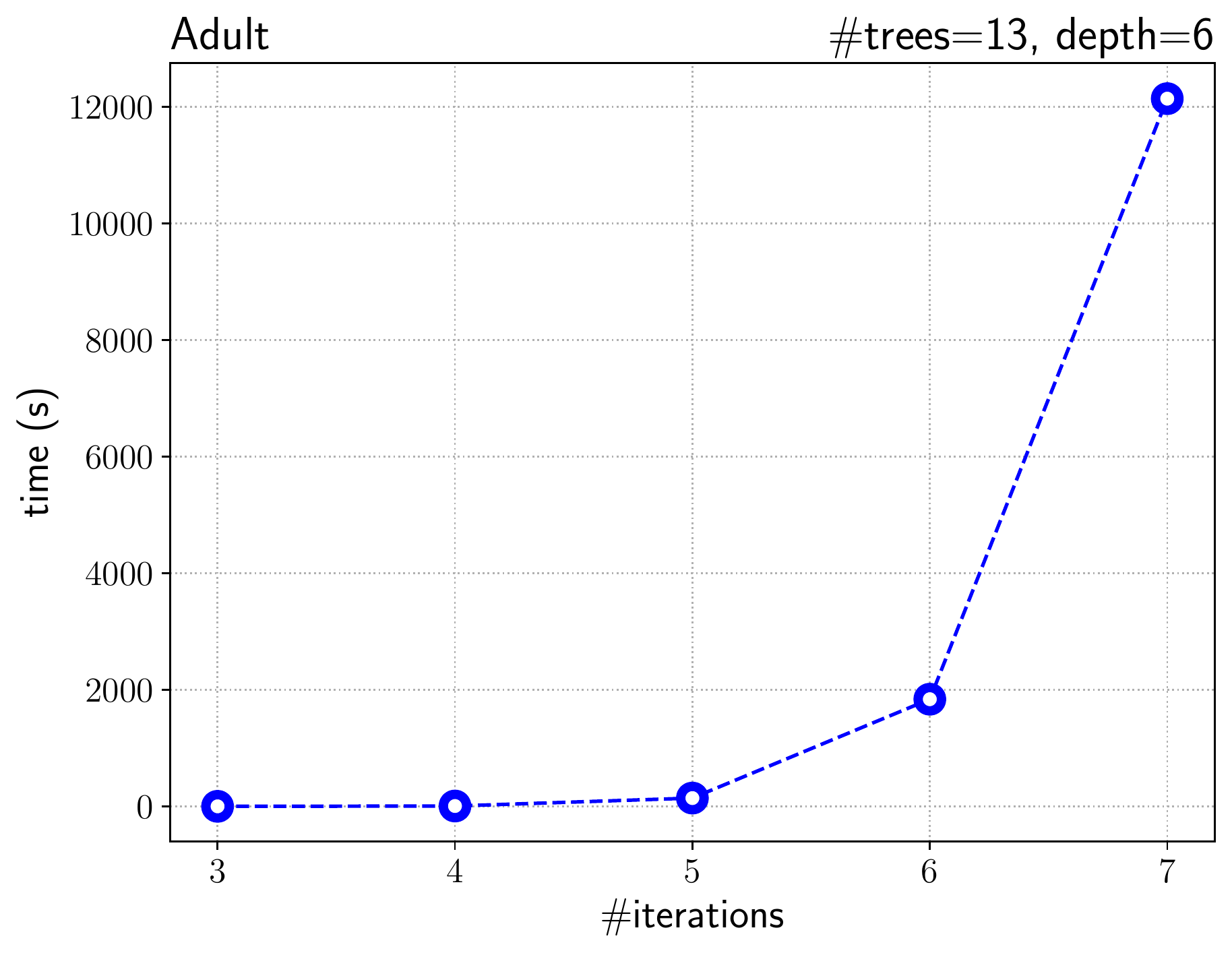}}
  \hfill
    \subfloat[]{\includegraphics[width=.33\textwidth]{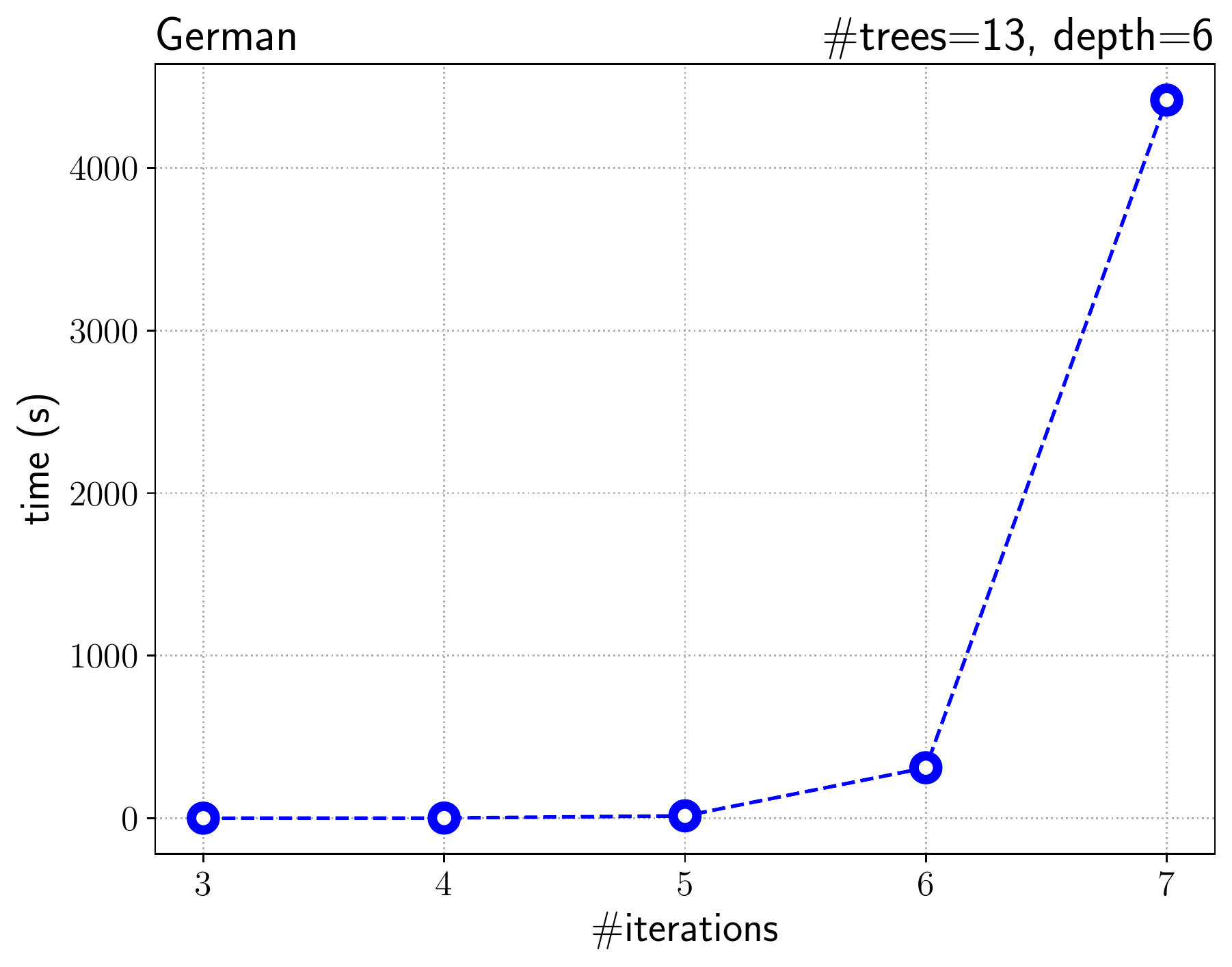}}
  \hfill
    \subfloat[]{\includegraphics[width=.33\textwidth]{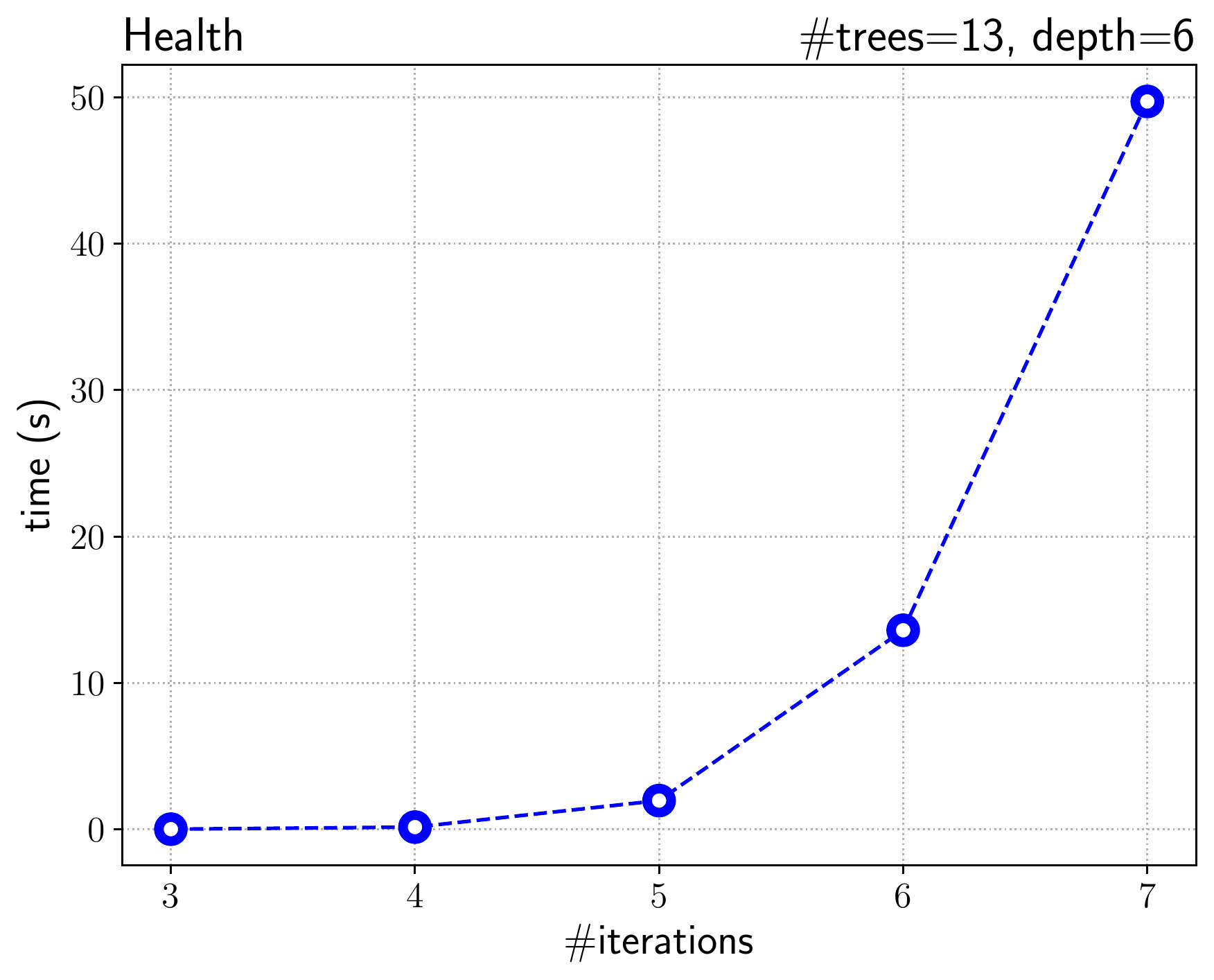}}
  \caption{Running times when varying the number of analysis iterations.}
  \label{fig: time}
\end{figure*}

We then show how the running times of the synthesis algorithm are affected by the number of trees in the ensemble, as well as by their depth. We first run the algorithm over ensembles including an increasing number of trees of at most depth six (Figure~\ref{fig: time-ntrees}) and then over ensembles including 13 trees of increasing maximum depth (Figure~\ref{fig: time-depth}), stopping the analysis after six iterations. \writtenbyLC{All the settings involving tree ensembles with at most 13 trees or trees with maximum depth six like in our previous experiments terminate in a matter of minutes, thus showing the practicality of our proposal.} Nevertheless, both plots show that the analysis times exhibit an exponential growth with respect to the complexity of the model, since the complexity of the analysis algorithm is exponential with respect to the number of items, which in turn depends on the number of features and thresholds occurring in the model. \writtenbyLC{Thus, the analysis time for more complex models, e.g., the tree ensembles with maximum depth seven trained on the Adult dataset, could significantly increase.} This further motivates the benefits of an iterative analysis approach as the proposed one, which allows one to leverage an early stopping criterion to collect partial (yet empirically precise) results even for cases where analysis convergence may turn out to be too expensive.

\begin{figure*}[t]
  \centering
    \subfloat[]{\includegraphics[width=.33\textwidth]{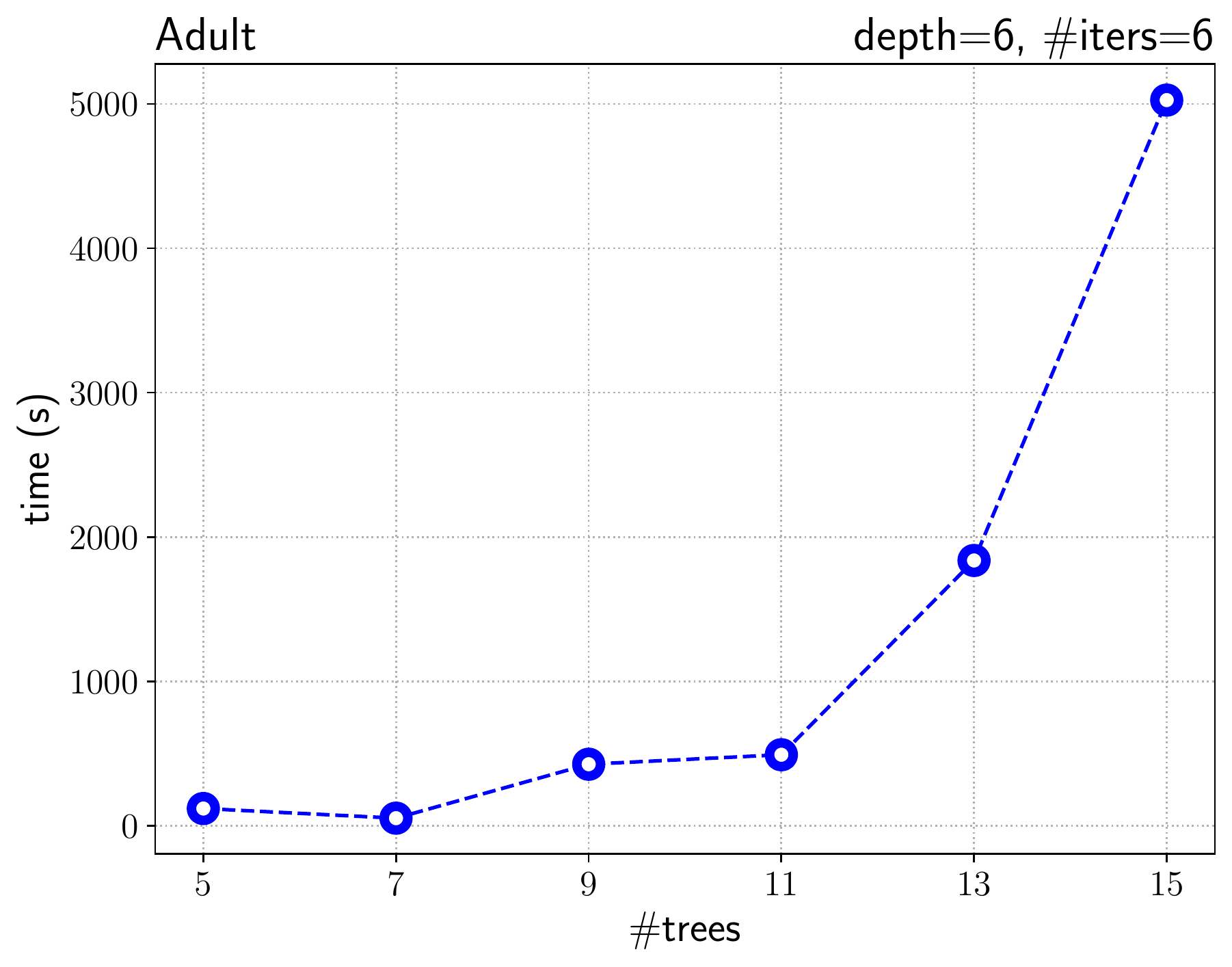}}
  \hfill
    \subfloat[]{\includegraphics[width=.33\textwidth]{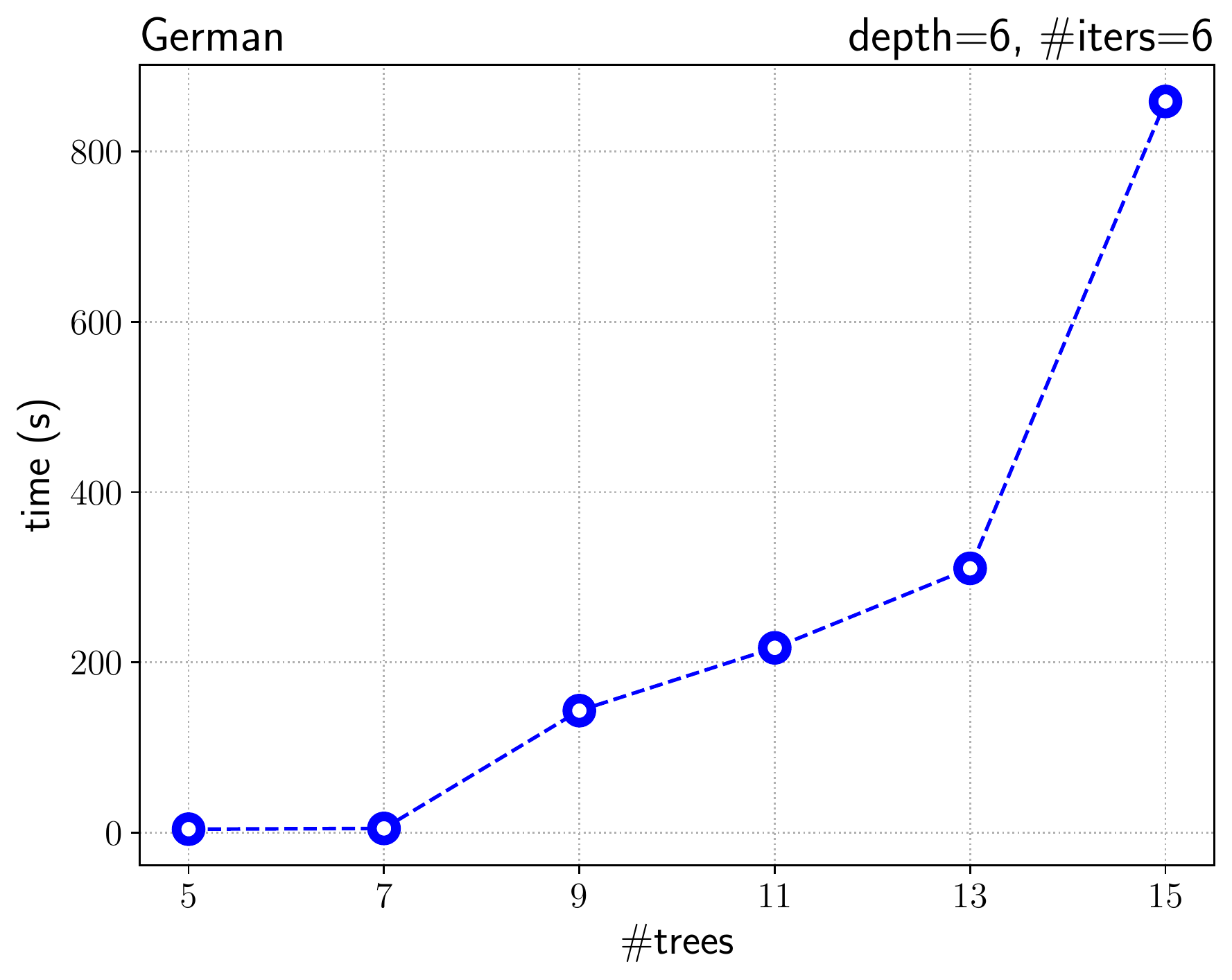}}
  \hfill
    \subfloat[]{\includegraphics[width=.33\textwidth]{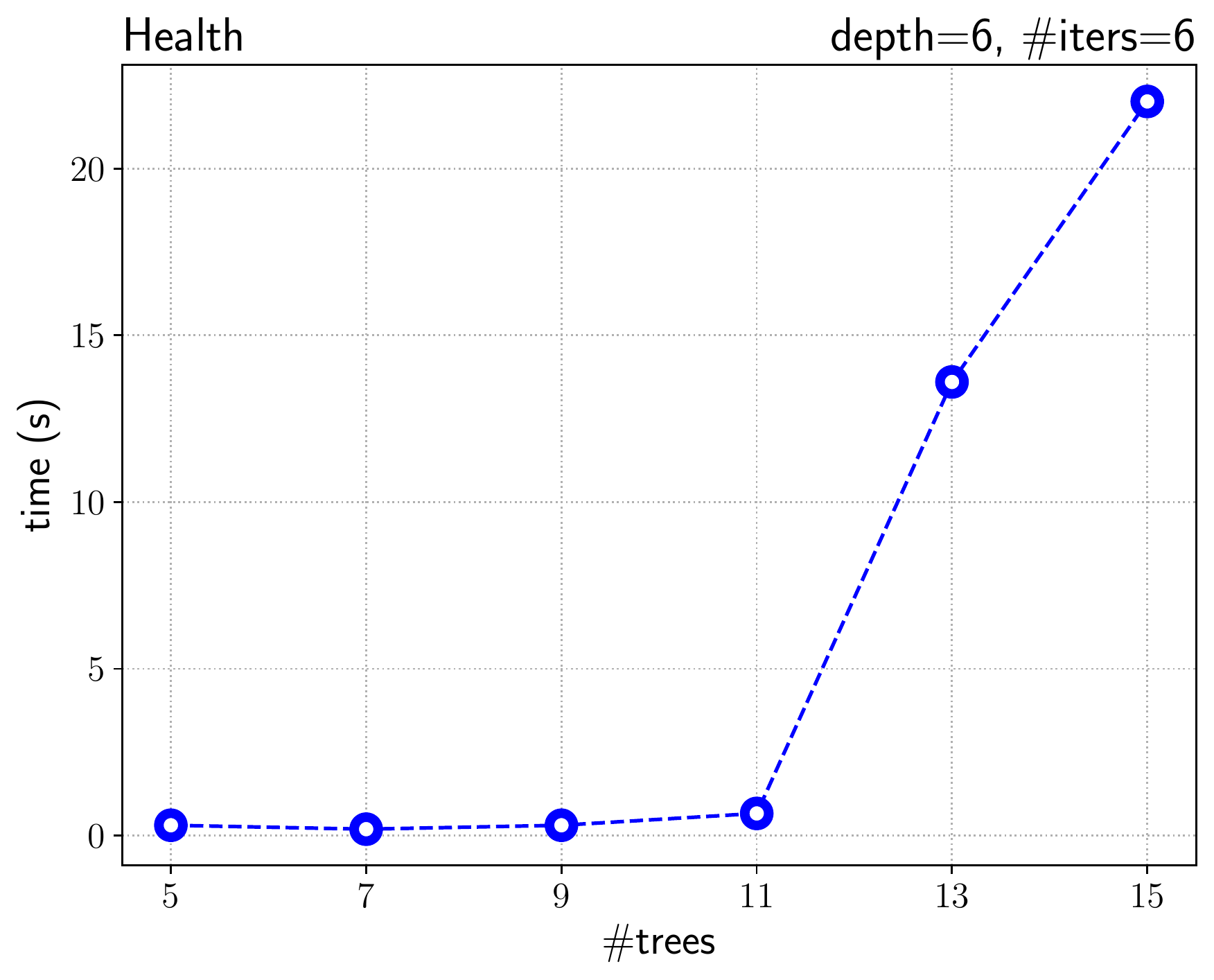}}
    \caption{Running times when varying the number of the decision trees in the ensemble.}
    \label{fig: time-ntrees}
    
    \centering
    \subfloat[]{\includegraphics[width=.33\textwidth]{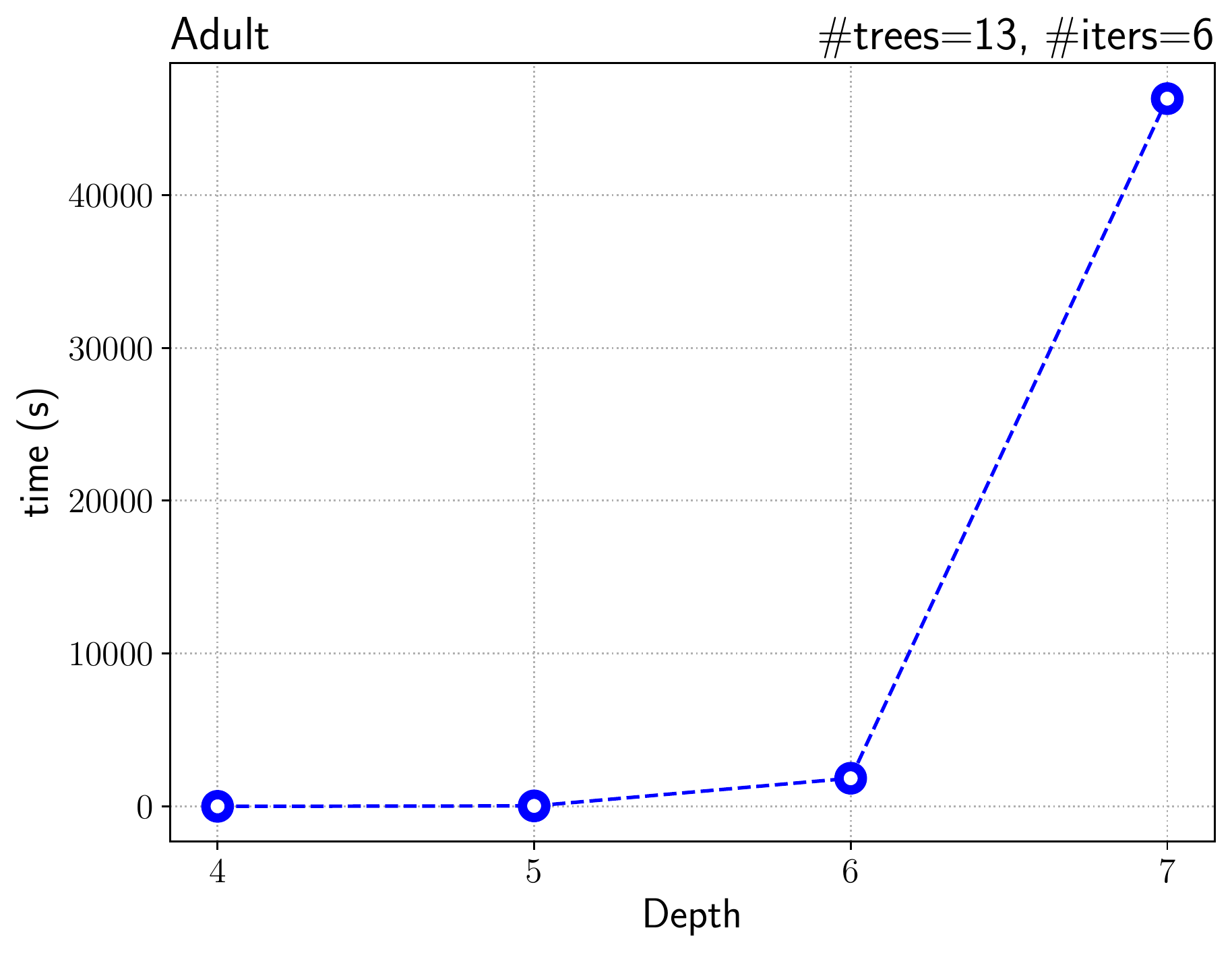}}
  \hfill
    \subfloat[]{\includegraphics[width=.33\textwidth]{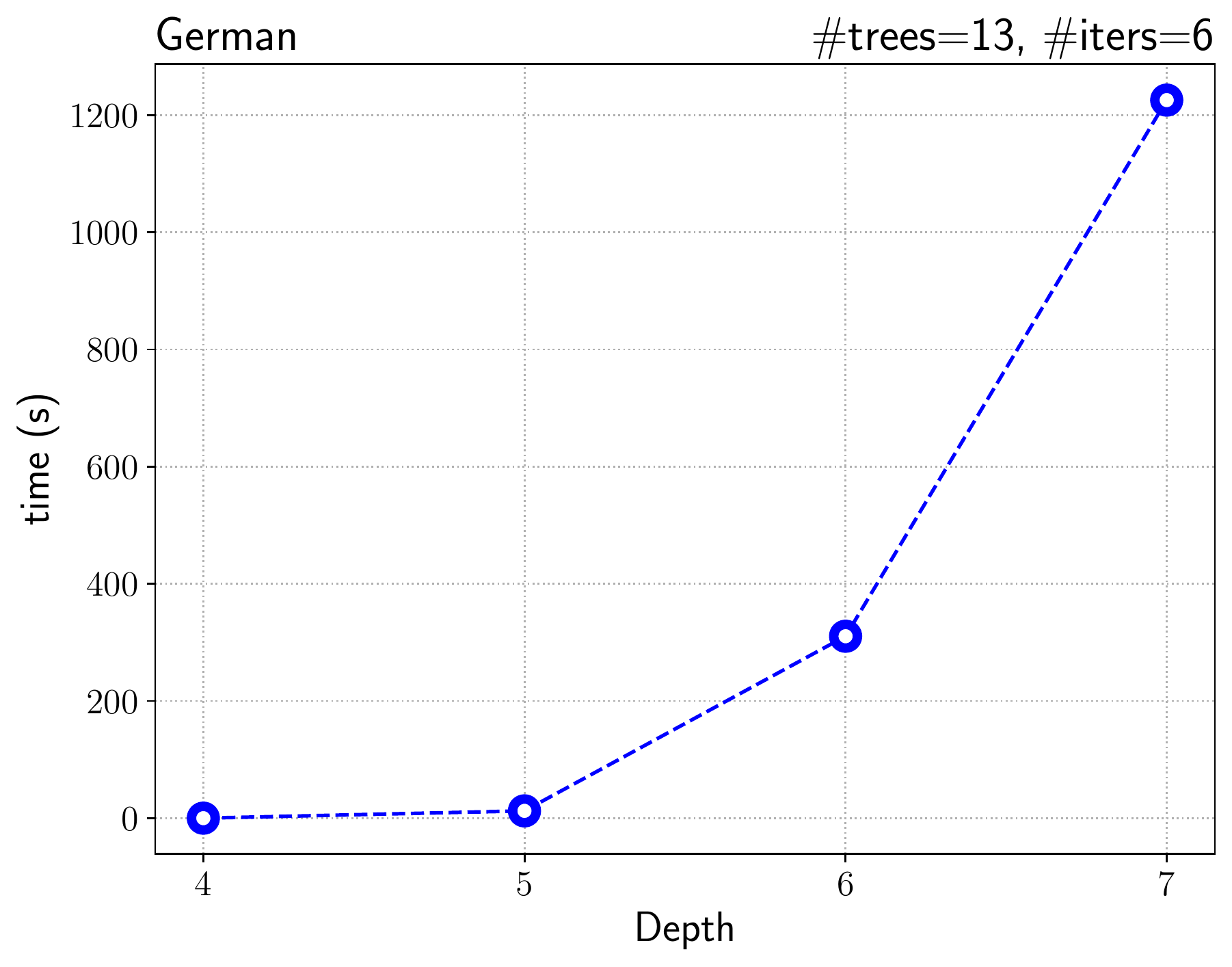}}
  \hfill
    \subfloat[]{\includegraphics[width=.33\textwidth]{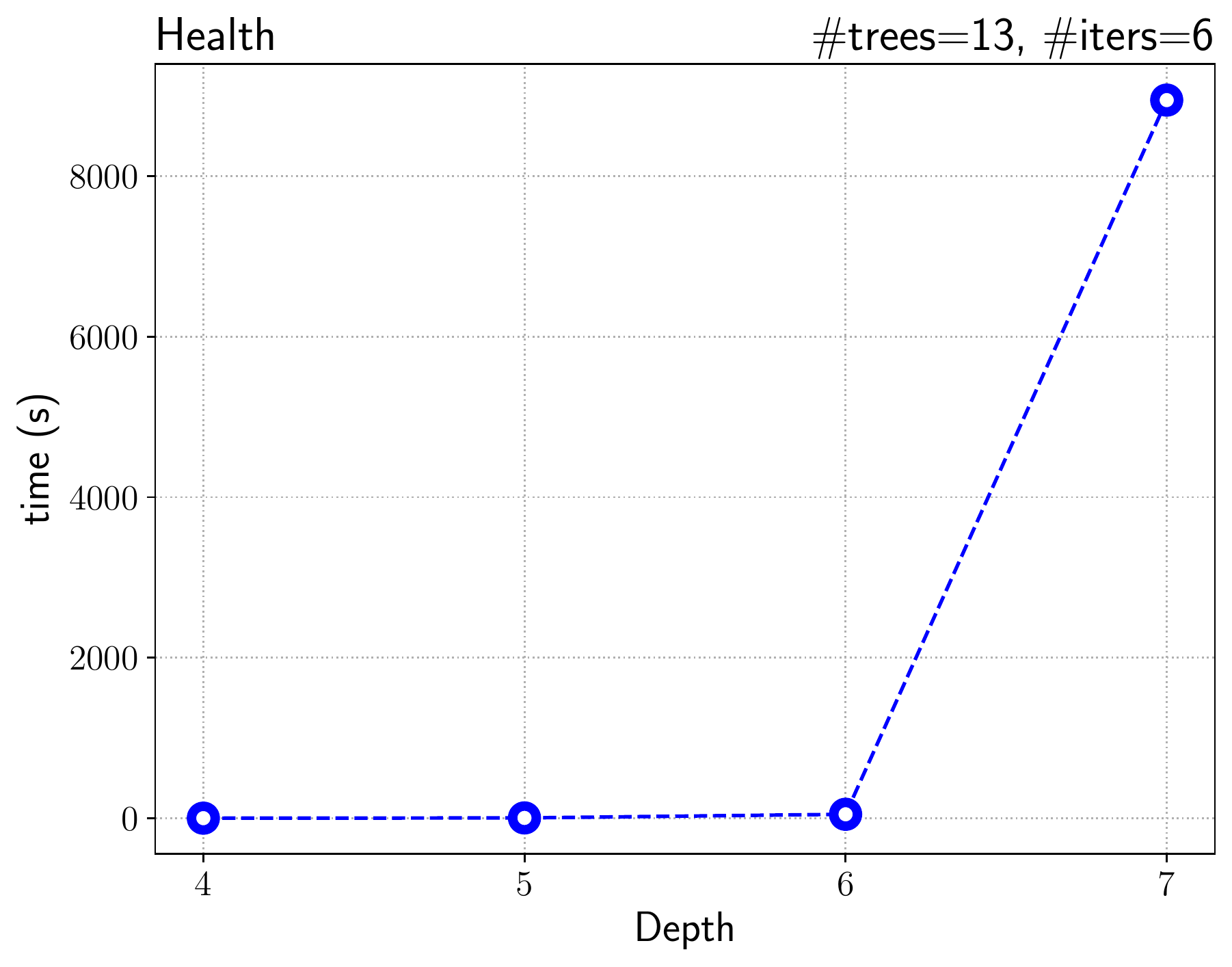}}
    \caption{Running times when varying the maximum depth of the decision trees in the ensemble.}
    \label{fig: time-depth}
\end{figure*}

\subsection{Summary}
Our previous experiments showed that a small number of iterations of the synthesis algorithm are sufficient to provide a precise and human-understandable characterization of the fairness guarantees provided by the analyzed classifiers. Specifically, for all datasets and models we were consistently able to obtain practically useful analysis results when terminating the analysis after six iterations. 

The results were \emph{precise}, because they matched the fairness guarantees provided by the classifier on both the test set and a set of synthetic instances randomly created to provide a more extensive picture of the entire feature space. The results were \emph{explainable}, because a small number of conditions of limited complexity turned out to be sufficient to largely characterize the fairness guarantees provided by the classifier, in particular over the test set. Finally, our algorithm is \emph{efficient} enough for practical adoption on the analyzed models, because all the experiments terminated in a matter of minutes when limiting the number of analysis iterations. Yet, we notice that our algorithm shows an exponential complexity with respect to the number of items, like the Apriori algorithm it is inspired from~\cite{AgrawalS94}. In our application setting, the number of items is correlated with the number of features and thresholds occurring in the model, hence we observe that \writtenbyLC{large models in terms of depth or number of trees} may eventually pose challenges to scalability. Nevertheless, we showed that limiting the number of analysis iterations is useful to collect meaningful results even for cases that may be too complicated to analyze up to convergence, which might be enough to mitigate scalability problems. We leave a more extensive evaluation on larger models and an additional optimization of our implementation to deal with them to future work.

\section{Related Work}
\label{sec:related}
We categorize related work in two broad research areas: fairness testing and fairness verification. We also briefly discuss prior work on fairness and explainability, which is a very recent research area that is getting traction.

\subsection{Fairness Testing}
Fairness testing estimates the fairness guarantees of a classifier by means of the automated generation of a number of test instances, which are fed to the model to identify potential room for discrimination. The first work on fairness testing we are aware of led to the development of the Themis tool~\cite{GalhotraBM17,AngellJBM18}. Themis uses a testing approach to identify instances which may suffer from causal discrimination, which is the same notion of fairness used in this paper. Its approach is based on a random generation of test instances, whose number is determined by a simple statistical test. Later work in the area like Aequitas improved the test instance generation approach by integrating the global random search of Themis with a perturbation-based local search~\cite{UdeshiAC18}. Even more recent work combined symbolic execution with local explanation techniques to further improve the effectiveness of fairness testing~\cite{AggarwalLNDS19}, proposed more advanced statistical tests~\cite{TaskesenBKN21} and improved the explainability of the testing results~\cite{BlackYF20}.

Contrary to our proposal, fairness testing performs an under-approximated analysis, i.e., it can find counter-examples to fairness, but it cannot prove fairness for specific classes of individuals. This is similar to what happens in software verification, where testing can be used to prove the presence of bugs, but not their absence. Testing thus plays an orthogonal role with respect to verification: when fairness guarantees cannot be proved for specific classes of individuals, one can rely on testing to identify counter-examples. Indeed, it would be interesting to study how our fairness verification approach could boost the effectiveness of fairness testing, in particular by directing it towards portions of the feature space where discrimination might potentially happen. We leave the investigation of this research idea to future work.

\subsection{Fairness Verification}
Fairness verification of ML models attracted a lot of attention by the community, as shown by the emergence of several recent surveys on the topic~\cite{abs-2010-04053,abs-2012-15816,MehrabiMSLG21}. Fairness is indeed a broad research area and several definitions of fairness have been proposed in the literature, each with its pros and cons: see~\cite{abs-1808-00023,VermaR18,MakhloufZP21} for a critical comparison of different fairness definitions. For the purposes of this paper, it is just worth noticing that existing fairness definitions can be broadly categorized in two classes: \emph{individual fairness} definitions, dictating that similar inputs must result in similar outputs, and \emph{group fairness} definitions, stipulating that a particular subset of inputs taken as a whole must be treated like a different subset. Of course, when pursuing formal verification, one has to stick to a specific notion of fairness, which is lack of causal discrimination in our case. This property belongs to the class of individual fairness and we choose it for different reasons: it is intuitive, popular, and provides global fairness guarantees beyond the test data. Since individual fairness and group fairness deal with different concerns and rely on rather different technical tools, we consider verification of group fairness as complementary to our work~\cite{AlbarghouthiDDN17,Bastani0S19,SunSDZ21,GhoshBM21}.

A key difference of our work with respect to the state of the art is its focus on tree-based models, which have been largely ignored by prior work. Indeed, the first work on the verification of individual fairness focused on linear models such as logistic regression and support vector machines~\cite{JohnVS20}. Most later work, instead, focused on neural networks and cannot be applied to decision tree ensembles~\cite{abs-2205-09927,UrbanCWZ20,RuossBFV20}. A general approach to fairness verification based on SMT solving was presented in~\cite{IgnatievCSHM20}. The approach can be applied to tree-based models by encoding them into SMT, however, it treats fairness just as a binary, unconditional property, i.e., a model is only considered to be fair when sensitive features never play any role in classification. This is too restrictive in practice, as shown by their experimental evaluation which only verified two out of 40 analyzed models as fair. Our verification approach is more expressive, because it automatically identifies sufficient conditions for fairness, defining specific classes of individuals where discrimination cannot happen; these classes may be small or large depending on the actual fairness guarantees of the model. Indeed, observe that the trivially true condition subsumes the simple verification problem considered in~\cite{IgnatievCSHM20}.

The only fairness verification approach designed for tree-based models we are aware of is presented in~\cite{RanzatoUZ21}. Their approach allows the verification of a local fairness property, which just predicates on a specific set of test instances, thus providing weaker fairness guarantees than the global fairness proofs offered by lack of causal discrimination. Also, their proposed approach just computes the causal discrimination score of the classifier over the test set, without providing any explanation of which individuals cannot suffer from discrimination. On the other hand, their approach was integrated into a fair training algorithm called FATT, which is an interesting avenue for future work we would like to pursue.

\subsection{Fairness and Explainability}
Reconciling fairness and explainability has been recognized as an important problem for specific application scenarios, such as algorithmic hiring~\cite{SchumannFMD20}. However, the area is still at its infancy~\cite{ZhouCH20} and there has been limited work on the topic so far.
\writtenbyLC{Recent work focused on proposing methods to generate feature-level bias explanations~\cite{abs-2010-07389, GhoshSW22} and counterfactual explanations~\cite{SharmaHG20}, investigate tools for evaluating ML models with respect to fairness~\cite{abs-2106-07483} and training ML models that are both explainable and fair~\cite{GrabowiczPM22, QiangLBZ22, abs-1910-02043}.}
Moreover, recent research also investigated explainable fairness for recommender systems~\cite{TanXG00Z21,GeTZXL0FGLZ22}.

The use of logical formulas for explainability purposes has also been investigated by the community. Prior work proposed approaches to use logical formulas as building blocks of logic-related models, i.e., rule lists~\cite{AngelinoLASR17} and decision sets~\cite{LakkarajuBL16}, that exhibit high explainability and accuracy. Moreover, since tree-based ensembles are difficult to explain because of their large number of trees, approaches for extracting rules providing an explainable and faithful description of a tree ensemble have been proposed~\cite{abs-2206-14359, Deng19, Friedman_2008, MashayekhiG17, HatwellGA20, BenardBVS21}. However, the idea of using logic formulas to explain the fairness guarantees provided by a ML model is novel to the best of our knowledge. In other words, a key difference of our work with respect to the state of the art is the target of the explanations, because our proposal aims to explain the fairness guarantees of a tree ensemble, not the outcome of its predictions.

\section{Conclusion}
\label{sec:conclusion}
In this paper, we presented a new global fairness verification approach for tree-based classifiers. Our approach synthesizes sufficient conditions for fairness, expressed as a set of traditional propositional logic formulas, which are readily understandable by human experts. The analysis is proved to be sound and complete. Extensive experimental results on public datasets show that the analysis is precise, easily understandable by human experts and efficient enough for practical adoption.

We foresee a few relevant directions for future work. First, we would like to leverage our verification approach as a powerful foundation to train tree ensembles satisfying global fairness properties. This seems feasible, because prior work showed how a local fairness verifier can be used to train locally fair models~\cite{RanzatoUZ21}. Moreover, we plan to integrate fairness verification and fairness testing by using the conditions generated by our analysis and SMT solving to effectively find counterexamples suffering from causal discrimination. Indeed, the conditions returned by our synthesis algorithm identify portions of the feature space that cannot include such counterexamples, so fairness testing can be made more effective by sampling only from different areas. Finally, we would like to explore the generalization of our approach to capture group fairness properties of tree-based models~\cite{SunSDZ21}.

\bibliographystyle{IEEEtran}
\bibliography{main}

\appendices

\section{Case Study: German Dataset}
\label{sec:appendix}

\begin{table}[t]
    \centering\caption{Top 5 logic formulas obtained for the German dataset (model with 13 trees of maximum depth six). The formulas are ordered by decreasing importance.}
    \begin{tabular}{c|l}
    Rank & \multicolumn{1}{c}{Formula} \\
    \hline
    1 & \begin{tabular}[l]{@{}l@{}} 
    status $=$ ``no checking account" \\
    $\wedge$ savings $\neq$ ``unknown/no savings account" \\
    $\wedge$ installment\_plans $=$ ``None''
    \end{tabular} \\
    \hline
    2 & \begin{tabular}[l]{@{}l@{}} 
    250.00 $\leq$ credit\_amount $\leq$ 7,464.50 \\
    $\wedge$ status $=$ ``no checking account" \\
    $\wedge$ (credit\_history $=$ ``no credits/all credits paid back duly" \\
    \quad $\vee$ credit\_history $=$ ``existing credits paid back duly till now" \\
    \quad $\vee$ credit\_history $=$ ``delay in paying off in the past")
    \end{tabular} \\
    \hline
    3 & \begin{tabular}[l]{@{}l@{}}
    status $=$ ``no checking account" \\
    $\wedge$ credit\_history $\neq$ ``critical account/other credits existing" \\
    $\wedge$ savings $\neq$ ``unknown/no savings account"
    \end{tabular} \\
    \hline
    4 & \begin{tabular}[l]{@{}l@{}}250.00 $\leq$ credit\_amount $\leq$ 7,464.50 \\
    $\wedge$ status $=$ ``no checking account" \\
    $\wedge$ credit\_history $=$ ``critical account/other credits existing" \\
    $\wedge$ installment\_plans $=$ ``None''
    \end{tabular} \\
    \hline
    5 & \begin{tabular}[l]{@{}l@{}}
    telephone $=$ False \\
    $\wedge$ status $=$ ``no checking account" \\
    $\wedge$ (credit\_history $=$ ``all credits at this bank paid back duly'' \\
    \quad $\vee$ credit\_history $=$ ``existing credits paid back duly till now'' \\
    \quad $\vee$ credit\_history $=$ ``delay in paying off in the past'')
    \end{tabular} \\
    \hline
    \end{tabular}
    \label{tab:top-german}
\end{table}

In this section, we present some examples of logic formulas generated by our synthesis algorithm to show that they provide effective explanations about the fairness of the model. 

We consider logic formulas synthesized for a tree ensemble with 13 trees of maximum depth six, trained on the German dataset. This dataset consists of 1,000 people who would like to get credit from a bank and the corresponding classification task consists in classifying their requests as high or low risk. We consider \textit{sex} as the sensitive feature, so the logic formulas provide us the description of subsets of people whose credit risk prediction does not change by flipping their sex. We present the top 5 most important formulas returned by the synthesis algorithm in Table~\ref{tab:top-german}, in decreasing order of importance. We immediately observe that the formulas are explainable, since they predicate on at most four features. We then comment three representative formulas below.

The first-ranked logic formula is interesting to examine, since it covers a subset of people that we expect to be highly represented also in the test set, because the formulas are ranked by the importance computed on the training set (see Section~\ref{sec:experiments}). In particular, it indicates that an individual that has no checking account, possesses a saving account whose amount is registered with the bank and has no installment plans will not receive a different credit risk depending on their sex. The formula suggests that the person's savings represent enough information for assigning the credit risk irrespective of the sex of the requester, at least when no information about the individual's checking account or other installment plans is available.

The second formula is interesting too. In particular, the formula indicates that the model does not discriminate by sex individuals who request a small credit amount (between 250 and $\sim$ 7,500 DM), do not have a registered checking account and do not present a critical credit history. The formula highlights a reasonable behavior of the ML classifier: when requesting small credit amounts, the presence of a good credit history already suffices for taking a decision, without relying also on the individual's sex.

Finally, the fourth formula is the only one in the top 5 formulas that predicates on individuals with a bad credit history. In particular, it explains that an individual who requests a small credit amount, has no checking account and other installment plans, but presents a critical credit history, is not discriminated by the classifier based on his sex. This is interesting because bad credit history should be the most important information when assessing a loan request. The reason why this only emerges for small loan requests might be twofold. First, assessing requests for small loans is expected to be easier, hence the corresponding proof of fairness is also easier and is established with a small number of analysis iterations. Moreover, the dataset includes many such requests, hence the importance of the formula increases and leads to its inclusion in the top formulas.

In conclusion, the synthesized logic formulas are useful to the analyst to conclude whether fairness guarantees can be provided for particular subgroups of instances in the domain of interest, not just for specific instances in the test set.

\end{document}